\documentclass{article}

\PassOptionsToPackage{authoryear}{natbib}

 \usepackage[preprint]{neurips_2026}

\usepackage[utf8]{inputenc} 
\usepackage[T1]{fontenc}    
\usepackage{hyperref}       
\usepackage{url}            
\usepackage{booktabs}       
\usepackage{amsfonts}       
\usepackage{nicefrac}       
\usepackage{microtype}      
\usepackage{xcolor}         
\usepackage{pifont}         

\usepackage{amsfonts}       
\usepackage{amsmath}        
\usepackage{nicefrac}       
\usepackage{array}          
\newcolumntype{C}[1]{>{\centering\arraybackslash}p{#1}}
\usepackage{microtype}      
\usepackage{graphicx}       
\usepackage{subcaption}     
\usepackage{verbatim}       
\usepackage{makecell}
\usepackage{pgfplots}
\pgfplotsset{compat=1.18}
\usepackage{caption}
\usepackage{multirow}
\usepackage{algorithm}
\usepackage{algpseudocode}
\usepackage{algorithmicx}
\usepackage{wrapfig}
\usepackage{bm}

\usepackage{amssymb}
\usepackage{mathtools}
\usepackage{amsthm}
\hypersetup{
    colorlinks=true, 
    linkcolor=gray,  
    citecolor=gray,  
    urlcolor=blue,   
    pdfborder={0 0 0}, 
}
\newtheorem{theorem}{Theorem}[section]
\newtheorem{proposition}[theorem]{Proposition}

\newtheorem{corollary}[theorem]{Corollary}

\title{Recursive Flow Matching}

\author{%
  Jiahe Huang$^1$\quad Sihan Xu$^2$\quad Sharvaree Vadgama$^1$\quad Rose Yu$^1$\\
  $^1$University of California, San Diego\quad
  $^2$University of Michigan\\
  \texttt{\{chh118, svadgama, roseyu\}@ucsd.edu} \quad
  \texttt{sihanxu@umich.edu}
}

\begin{document}

\maketitle

\begin{abstract}
    Generative models have emerged as a powerful paradigm for solving physics systems and modeling complex spatiotemporal dynamics. However, achieving high physical accuracy without incurring high computational cost remains a fundamental challenge, as existing approaches face a critical speed-fidelity trade-off.
    In this work, we introduce \textbf{Recursive Flow Matching} (\textit{RecFM}), a generative framework for forecasting complex spatiotemporal dynamics.  RecFM enforces self-consistency to align trajectories across discretization scales, reducing discretization errors and improving performance across metrics for physics-based tasks. To our knowledge, this is the first method to achieve high-fidelity one- and few-step (\textit{2-4 step}) dynamic generation for scientific systems with performance comparable to state-of-the-art multi-step solvers. 
    Across challenging scientific benchmarks, RecFM achieves up to a $20\times$ speedup over leading diffusion-based emulators while improving predictive accuracy. Furthermore, RecFM reduces mean squared error by over $15\%$ compared to vanilla flow matching, offering a scalable and efficient solution for real-time scientific emulation.
    Project page: \href{https://jhhuangchloe.github.io/RecFM/}{jhhuangchloe.github.io/RecFM/}.
\end{abstract}

\section{Introduction}\label{sec:intro}

Predicting the evolution of physical systems is a fundamental challenge in scientific computing, with applications ranging from fluid dynamics to climate modeling and weather forecasting.
Traditional numerical solvers provide high-fidelity solutions \cite{dhatt2012finite, cantwell2015nektar++}, but are typically computationally expensive and impractical for real-time or large-scale deployment. These limitations motivate the need for data-driven approaches that can efficiently model complex, high-dimensional dynamics.
With advancements in scientific machine learning approaches like neural operators \cite{kovachki2023neural, li2020fourier, lu2021learning} and PINNs \cite{raissi2019physics, penwarden2022multifidelity} are widely used to simulate systems described by partial differential equations (PDEs). However, in real-world applications, these governing equations are frequently incomplete, computationally prohibitive, or challenging to formulate for complex and stochastic systems such as climate dynamics. 

Recent advances in generative modeling provide a powerful framework for learning high-frequency data distributions tailored to scientific applications, addressing key challenges in molecular design \cite{abramson2024accurate, shen2025simultaneous}, material generation \cite{zeni2025generative}, and climate modeling \cite{duncan2025samudrace, watt2025ace2}. In these fields, the ability of generative models to quantify uncertainty and manage sparse or irregular measurements offers significant advantages over traditional deterministic methods. Especially in computational physics, generative methods have been shown to reconstruct spatiotemporal dynamics from limited observations, such as turbulent fluid flow or atmospheric models, effectively bridging the gap between inductive statistical learning and deductive physical laws \cite{cachay2025elucidated, huang2024diffusionpde, ruhling2023dyffusion, zhuang2025spatially}. Nevertheless, deploying these models for accurate dynamical prediction remains challenging, as they must balance efficiency with the preservation of physical fidelity over time.

A key limitation of diffusion-based models is their inherently iterative inference procedure, which requires tens to hundreds of sequential denoising steps to produce high-quality predictions \cite{ho2020denoising, karras2022elucidating, song2020denoising, nichol2021improveddenoisingdiffusionprobabilistic}. This results in significant computational overhead, especially for time-dependent simulations. To address this issue, continuous normalizing flows (CNFs) \cite{mathieu2020riemannian} and flow matching (FM) \cite{chen2024flowmatchinggeneralgeometries, geng2025mean, lipman2022flow} have emerged as efficient alternatives, learning continuous vector fields that define probability paths without requiring simulation during training. While these approaches reduce the number of required function evaluations, a fundamental trade-off remains: reducing the number of inference steps often leads to degraded accuracy and instability, particularly in long-term dynamical rollouts.

To further accelerate these systems, a wide range of approaches have been proposed, including consistency models and distillation-based methods \cite{tauberschmidt2025physics, xu2023cyclenet}. Consistency models, such as Shortcut Diffusion \cite{frans2024one}, introduce self-consistency constraints that enable direct mapping along the probabilistic path in a single step, while distillation techniques aim to compress multi-step generation into an efficient student model \cite{song2024multi}. However, a key challenge in these approaches is preserving the spectral richness and spatiotemporal fidelity of physical fields, as aggressive step reduction often smooths out high-frequency structures that are critical for accurate scientific simulations \cite{xu2025understanding}.
These limitations highlight the need for a framework that can achieve efficient few-step (typically \emph{at most four steps}) generation while maintaining trajectory fidelity and stability. 

\begin{figure}[t]
    \centering
    \begin{subfigure}[t]{0.45\linewidth}
        \centering
        \includegraphics[width=\linewidth]{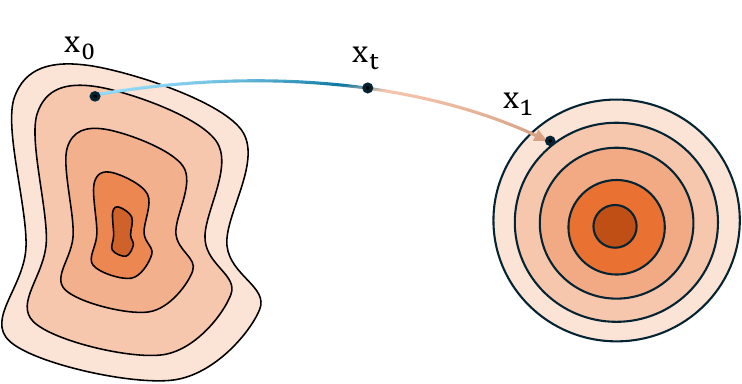}
        \caption{Flow Matching}
        \label{fig:fm}
    \end{subfigure}
    \hfill
    \begin{subfigure}[t]{0.45\linewidth}
        \centering
        \includegraphics[width=\linewidth]{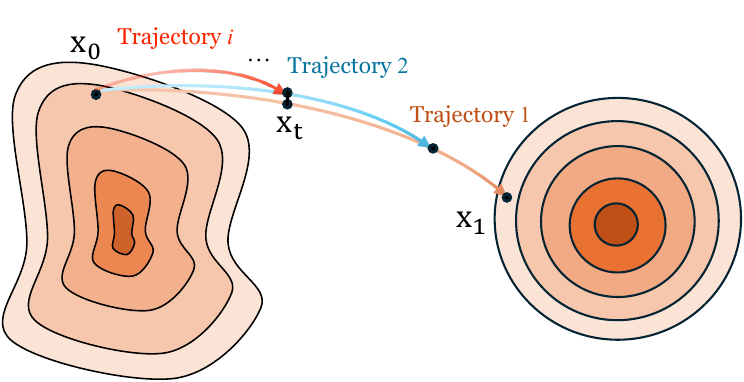}
        \caption{Recursive Flow Matching}
        \label{fig:rfm}
    \end{subfigure}
    \caption{\textbf{Comparison of flow matching paradigms.} 
    (a) Flow Matching (FM) learns a direct trajectory that transports samples from the data distribution ($x_0$) to the noise distribution ($x_1$). 
    (b) Recursive Flow Matching (RecFM) augments this with recursively scaled trajectories (brown, blue, and red arrows) that intersect at shared spatial states ($x_t$), enabling cross-scale trajectory alignment and consistency training along the flow.}
    \label{fig:comparison_fm_rfm}
    \vspace{-1em}
\end{figure}

To address these challenges, we introduce \textbf{Recursive Flow Matching} (\textit{RecFM}), a generative framework for stable and efficient modeling of dynamical systems. Instead of relying on a single discretized trajectory, RecFM recursively models a family of trajectories spanning different inference-time traversal scales and enforces consistency among them.
In particular, trajectories at different scales are coupled by aligning states that correspond to the same underlying point along the path, ensuring that predictions remain coherent across discretizations. This multi-scale coupling provides additional supervision and improves stability in one- or few-step regimes.
Our main contributions include:
\vspace{-0.5em}
\begin{itemize}
\item \textbf{Recursive Flow Matching:} A novel flow matching framework for forecasting complex physical dynamics, 
enabling a unified treatment of systems governed either by explicit PDE formulations or by implicitly learned data-driven dynamics.
\vspace{-0.2em}
\item \textbf{Multi-Scale Trajectory Alignment:} A mechanism that enforces consistency of trajectories across sampling scales, stabilizing dynamical rollouts and mitigating error accumulation over multiple inference steps.
\vspace{-0.2em}
\item \textbf{High-Efficiency Emulation:} We validate our approach on both simulated and real-world physical dynamics prediction benchmarks, achieving state-of-the-art accuracy with substantially fewer sampling steps. 
\end{itemize}
\vspace{-0.7em}

\section{Background}
In this section, we introduce the necessary background for our proposed \textit{RecFM}. We briefly review generative and trajectory flow matchings, which form the core building blocks of our framework.
\vspace{-.7em}
\subsection{Flow Matching}

Flow Matching (FM) \cite{lipman2022flow} is a simulation-free paradigm for training Continuous Normalizing Flows by regressing onto a target vector field. Let $p_0$ denote the target data distribution and $p_1$ denote a tractable source distribution (\textit{e.g.}, a standard Gaussian). FM seeks to learn a time-dependent vector field $v_t(x; \theta): \mathbb{R}^d \to \mathbb{R}^d; t\in [0,1]$ that defines a probability path $p_t$ connecting $p_0$ and $p_1$. The transformation of a sample $x_0 \sim p_0$ to $x_1 \sim p_1$ is governed by the ordinary differential equation (ODE):
\begin{equation}
    \frac{d\psi_t(x)}{dt} = v_t(\psi_t(x), t), \quad \psi_0(x) = x_0
\end{equation}
where $\psi_t$ represents the flow map. To ensure tractability, Conditional Flow Matching (CFM) utilizes a per-sample regression objective:
\begin{equation}
    \mathcal{L}_{\text{CFM}}(\theta) = \mathbb{E}_{t \sim U, x_0 \sim p_0, x_1 \sim p_1} [\|v_t(x_t, t; \theta) - u_t(x_t | x_0, x_1)\|^2]
\end{equation}
where $u_t(x_t | x_0, x_1)$ is the conditional velocity field. A prevalent choice is the Optimal Transport (OT) path, which utilizes linear interpolation $x_t = (1-t)x_0 + tx_1$ to yield a constant target velocity $u_t(x_t | x_0, x_1) = x_1 - x_0$. 
Although this formulation fully specifies the generative process, its practical performance is largely determined by the structure of the induced trajectories, motivating a closer examination of trajectory design.

The choice of trajectory (\textit{i.e.}, the transport map) plays a key role in determining sampling efficiency and stability. Approaches focusing on Trajectories of Flow Matching  \cite{zhang2025trajectoryflowmatchingapplications, islam2025longitudinalflowmatchingtrajectory}, to parameterize the drift and diffusion terms to model stochastic and irregularly sampled time series. From a physical perspective, such trajectories can be interpreted as approximations of the underlying system dynamics, where geometric simplicity contributes to stable and accurate generation. Yet existing methods fail to maintain consistency across discretization scales, compromising both accuracy and physical fidelity.

\vspace{-.7em}
\subsection{Self-Consistency and the Flow Map}
\vspace{-0.3em}
To overcome the iterative bottleneck of FM, recent work has introduced the self-consistency property \cite{frans2024one, xu2023inversion}. For a flow map $\mathbf{X}_{s,t}: \mathbb{R}^d \to \mathbb{R}^d$ that transports a state from time $s$ to time $t$, self-consistency requires that all points along a single trajectory map to the same endpoint. This is formally described by the semigroup condition:$$\mathbf{X}_{u,t}(\mathbf{X}_{s,u}(x)) = \mathbf{X}_{s,t}(x)$$for all $s, u, t$ such that $0 \le s \le u \le t \le 1$ where $u$ is an intermediate timestep. In ``one-step'' models, a consistency function $f_\theta(x_t, t)$ is trained to satisfy $f_\theta(x_t, t) = x_1$ for all $t \in [0,1]$. 
By executing this condition, the model ensures that the generated path remains unchanged, whether it is traversed in a single large step or in multiple smaller increments. This is an advanced regularization that can ``straighten'' the ODE trajectory and minimize the common truncation errors in accelerator solvers.

\begin{wrapfigure}{r}{0.27\linewidth}
    \vspace{-4em}
    \centering
    \includegraphics[width=0.78\linewidth]{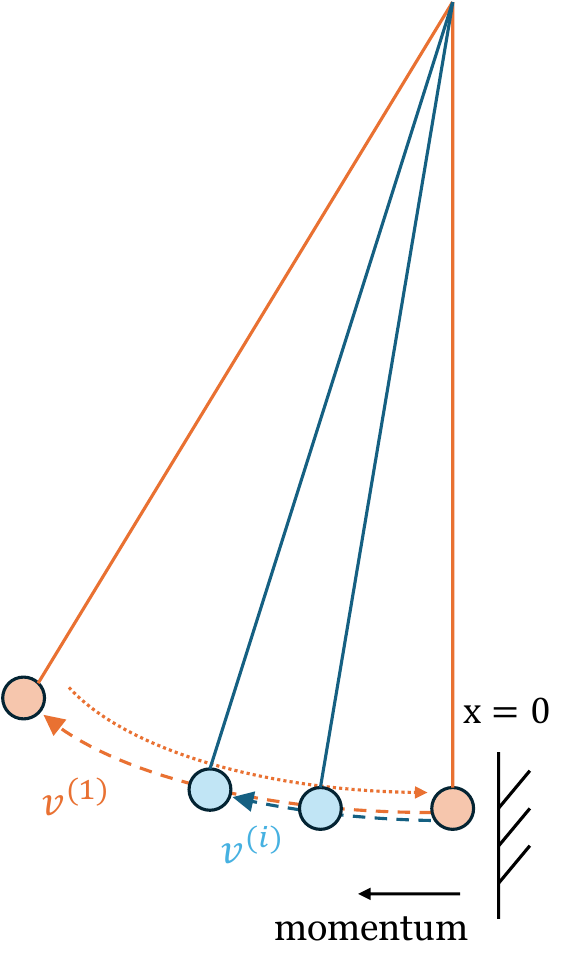}
    \vspace{-0.5em}
    \caption{Pendulum trajectories and velocities for the primary trajectory ($v^{(1)}$, \textcolor{orange}{orange}) and attenuated trajectories ($v^{(i)}$, $i>1$, \textcolor{cyan}{blue}).}
    \label{fig:pendulum}
    \vspace{-1.0 em}
\end{wrapfigure}

\vspace{-0.7em}
\section{Recursive Flow Matching}\label{sec:methods}

\label{subsection:wall_bouncing}

\vspace{-0.3em}
We draw inspiration from the recursive movement of an ideal\footnote{We do not consider energy loss due to friction or drag forces.} wall-bouncing pendulum to design our method, \textit{RecFM}. Below, we introduce the pendulum model, followed by the secondary trajectory formulation and the updated loss function for RecFM.

\vspace{-0.5em}
\subsection{Physics Intuition}
\vspace{-0.3em}

Let's consider the classical physics toy problem of a 1D wall-bouncing pendulum, illustrated in Figure~\ref{fig:pendulum}. Let 
$x(t)$ and $v(t)$ be the position and velocity of the pendulum at time $t$
respectively. Away from the wall at $x=0$, the pendulum travels at \textbf{constant speed} governed by: 
\begin{equation*}
    \dot{x}(t) = v(t), \quad \dot{v}(t) = 0
\end{equation*}
A bounce occurs every time the pendulum strikes the wall ($x = 0$), resulting in a set of trajectories. At each collision, 
the velocity reverses direction and its magnitude is reduced, with a fraction $1 - \alpha^2$ of the kinetic energy lost, where $\alpha \in [0, 1]$ is the velocity retention coefficient. For simplicity, we consider velocities along a fixed direction (\textit{e.g.}, from the wall toward the turning point), so that only their magnitudes are tracked across bounces.
Let $v^{(i)}$ denote the velocity magnitude immediately after the $i$-th bounce. The collision update rule is:
\begin{equation}
    v^{(i+1)} = \alpha\, v^{(i)}.
\end{equation}

We assume a constant half-cycle duration across scales, consistent with small-angle dynamics, so that amplitude shrinks proportionally with velocity after each bounce. While not strictly physical, this yields a simple and tractable parameterization across trajectories.

After $D-1$ collisions, we obtain a family of trajectories $\{v^{(i)}\}_{i=1}^D$ with progressively attenuated velocities. Writing $\bm{v}^* := v^{(1)}$ and $\alpha^{(i)} := \alpha^{i-1}$, we obtain the scaling relation
\begin{equation}
    v^{(i)} = \alpha^{(i)}\, \bm{v}^*.
\end{equation}
This velocity consistency defines a natural supervision signal for our multi-scale objective.

Figure~\ref{fig:comparison_fm_rfm} illustrates the key idea behind our approach. 
While vanilla flow matching learns a single trajectory between $x_0$ and $x_1$, 
we extend this formulation by introducing additional trajectories at different 
scales. Building on this intuition, we propose \textbf{Recursive Flow Matching} 
(\textit{RecFM}), which enforces consistency across these trajectories.

\vspace{-0.5em}
\subsection{RecFM Algorithm}
\vspace{-0.5em}
\paragraph{Formulation.}
Given a data sample $x_0 \sim p_0$ and a noise sample $x_1 \sim p_1$, we define the standard linear interpolant
\begin{equation}
x_t = (1-t)\,x_0 + t\,x_1, \qquad \bm{v}^\ast = x_1 - x_0.
\end{equation}
The velocity network $v_\theta(x, t, \alpha)$ is conditioned on both
time $t$ and scale $\alpha$, so that a single model represents the
entire family of trajectories. 

Consider a recursive formulation with depth $D$, where $D$ trajectories are defined by time-scale pairs $\{(\tau^{(i)}, \alpha^{(i)})\}_{i=1}^{D}$. The rescaled time is defined as $\tau^{(i)} = t / \alpha^{(i)} \in [0,1]$, with $\alpha^{(1)} = 1$ and $\tau^{(1)} = t$. Let $\hat{v}^{(i)} = v_\theta(x_t, \tau^{(i)}, \alpha^{(i)})$ denote the predicted velocity of the $i$-th trajectory.
Under this alignment, all trajectories pass through the same spatial point $x_t$, yielding the cross-scale consistency relation
\begin{equation}
    \hat{v}^{(i+1)} = \alpha\, \hat{v}^{(i)}.
\end{equation}
This shared spatial point, which is visited by trajectories of different scales at correspondingly aligned times, is the structural property RecFM exploits.

\vspace{-0.8em}
\paragraph{Algorithm.}
We present Algorithm~\ref{alg:train}, which trains a velocity network on $D$ recursive trajectories passing through the same point $x_t$: a primary trajectory (\textit{i.e.}, $i=1$) that learns the standard noise-to-data velocity $x_1 - x_0$, and $D-1$ time-rescaled secondary trajectories parameterized by $\alpha^{(i)}$, whose target velocities are given by $\alpha^{(i)}(x_1 - x_0)$, inspired by the wall-bouncing dynamics in Section~\ref{subsection:wall_bouncing}.

\begin{algorithm}[b]
\caption{Recursive Trajectory Training with Consistency Alignment}
\label{alg:train}
\begin{algorithmic}[1]
\Require Data distribution $p_0$, Noise distribution $p_1$
\Require Velocity network $v_\theta(x,t,\alpha)$, recursion depth $D$
\Require Consistency weight $\lambda$, total training iterations $N$

\For{iteration $n = 1$ to $N$}
    \State Sample $x_0 \sim p_0$ and $x_1 \sim p_1$ \Comment{Data and noise samples}
    \State Sample $t \sim \mathcal{U}(0,1)$ and $\alpha \sim \mathcal{U}(t,1)$ \Comment{Primary trajectory time and base recursion scale}
    
    \State $\bm{v}^* \gets x_1 - x_0$ \Comment{Ground-truth primary velocity}
    \State $x_t \gets (1-t)x_0 + t x_1$ \Comment{Shared spatial point}

    \For{$i = 1$ to $D$}
        \State $\alpha^{(i)} \gets \alpha^{i-1}$  \Comment{Recursive trajectory scale}
        \State $\tau^{(i)} \gets t / \alpha^{(i)}$ \Comment{Aligned trajectory time}
        \State $\hat{v}^{(i)} \gets v_\theta(x_t, \tau^{(i)}, \alpha^{(i)})$ \Comment{Predicted trajectory velocity}
        \State $\mathcal{L}_{\text{traj}}^{(i)} \gets \|\hat{v}^{(i)} - \alpha^{(i)} \bm{v}^*\|_2^2$ \Comment{Trajectory supervision}
    \EndFor

    \For{$i = 2$ to $D$}
        \State $\mathcal{L}_{\text{cons}}^{(i)} \gets \|\hat{v}^{(i)} - \alpha^{(i)} \hat{v}^{(1)}\|_2^2$ \Comment{Cross-scale consistency}
    \EndFor

    \State $\mathcal{L}_{\text{total}} \gets \sum_{i=1}^{D} \mathcal{L}_{\text{traj}}^{(i)}
    + \lambda \sum_{i=2}^{D} \mathcal{L}_{\text{cons}}^{(i)}$

    \State Update $\theta$ using $\nabla_\theta \mathcal{L}_{\text{total}}$
\EndFor
\end{algorithmic}
\end{algorithm}

\vspace{-0.5em}
\paragraph{Training Objective.}
To enforce alignment across trajectory scales, we build on the recursive formulation above. 
The overall training objective aggregates supervision across all scales and enforces consistency with the primary trajectory: \vspace{-0.1em}
\begin{equation}\label{eq:training_objective}
\begin{aligned}
\mathcal{L}_{\text{total}}
= \sum_{i=1}^{D} \mathcal{L}_{\text{traj}}^{(i)} 
&+ \lambda \sum_{i=2}^{D} \mathcal{L}_{\text{cons}}^{(i)} \\
\quad \text{where } 
\mathcal{L}_{\text{traj}}^{(i)}
= \left\|\hat{v}^{(i)} - \alpha^{(i)}\bm{v}^*\right\|_2^2,& \;
\mathcal{L}_{\text{cons}}^{(i)}
= \left\|\hat{v}^{(i)} - \alpha^{(i)} \hat{v}^{(1)}\right\|_2^2.
\end{aligned}
\end{equation}

\vspace{-0.9em}
\paragraph{Inference Sampling.}

Inference in RecFM is conducted by numerically solving the ODE defined by the learned velocity field $\hat{v}_\theta(x_t, t, \alpha)$. For single-step generation, using a first-order Euler step of size $h$, RecFM maps a noise sample $x_1 \sim p_1$ to the data manifold in one function evaluation:
\begin{equation}
    x_0 \approx x_1 - h\,\hat{v}_\theta(x_1, 1, 1)
\end{equation}
where $h = 1$, corresponding to integrating the trajectory over the full time horizon.

For multi-step generation, discretizing the trajectory into $K$ steps $1 = t_0 > \dots > t_K = 0$ with step sizes $h_k = t_{k-1} - t_k$, we iteratively update:
\begin{equation}
    x_{t_k} = x_{t_{k-1}} - h_k\,\hat{v}_\theta(x_{t_{k-1}}, t_{k-1}, 1), \quad k = 1, \dots, K.
\end{equation}
By enforcing cross-scale velocity consistency during training, RecFM learns trajectories that remain stable under larger integration steps, enabling accurate few-step generation.

\vspace{-0.5em}
\subsection{Theoretical results}
\vspace{-0.3 em}
We present Theorem \ref{thm:truncation} to show that adding recursive trajectories and cross-scale trajectory consistency loss accelerates the convergence of RecFM.
\begin{theorem}[Truncation Error Reduction via Trajectory Straightening]
\label{thm:truncation}
Let $\hat{v}_\theta(x, t, \alpha)$ be the predicted velocity and $\mathbf{a}(x,t) = \partial_t\, v_{\theta}(x,t,1) + (\nabla_x v_{\theta})\, v_{\theta}(x,t,1)$ denote the trajectory acceleration. The $K$-step Euler generation error with step size $h = 1/K$ satisfies
    \begin{equation}\label{eq:euler}
        \left\|\psi_1 - \hat{\psi}_1\right\| \;\leq\; \frac{h}{2}\,\frac{e^{L} - 1}{L}\,\sup_{t\in[0,1]} \left\|\mathbf{a}(\psi_t, t)\right\|,
    \end{equation}
    where $L = \sup_t \|\nabla_x v_{\theta}(\cdot, t, 1)\|_{\mathrm{op}}$. The acceleration decomposes into a temporal component and an advective term, $\mathbf{a} = \partial_t v_\theta + (\nabla_x v_\theta)\,v_\theta$. Minimizing $\mathcal{L}_{\textup{cons}}$ enforces the cross-scale consistency condition
    \begin{equation}\label{eq:pde}
        t\,\partial_t\, v_\theta(x, t, 1) \;+\; v_\theta(x, t, 1) \;=\; \partial_\alpha\, v_\theta(x, t, 1),
    \end{equation}
    which constrains $\|\partial_t v_\theta\|$ and thereby reduces $\|\mathbf{a}\|$, tightening~\eqref{eq:euler}. 
\end{theorem}

\vspace{-1em}
\begin{proof}
    See Appendix \ref{appendix:proof}.
\end{proof}

\vspace{-0.7em}
\paragraph{Why does RecFM work?}
\leavevmode\vspace{0.2em}\\ 
A given interpolated state $x_t$ lies on infinitely many trajectories indexed by $\alpha$. Vanilla FM exploits only one of them, providing a single regression target $\bm{v}^\ast$ per sample. RecFM uses every $(\tau, \alpha)$ pair as an independent supervisory signal for the \emph{same} underlying directional quantity $x_1 - x_0$ at the \emph{same} spatial point, while following the marginal distribution (Theorem \ref{thm:marginal}).
This functions as data augmentation in the conditioning space of the network and is particularly valuable in the one-step regime, where generation quality depends entirely on a single evaluation $v_\theta(x_0, 0, 1)$. By coupling predictions across scales, RecFM enriches the gradient signal at every training point and removes the warm-up phase typically required by shortcut or consistency-style training (Appendix \ref{app:shortcut-recfm}).

\vspace{-0.5em}
\section{Related Work}
\vspace{-0.5em}

\paragraph{Neural PDE Solvers and Physics-Informed Learning.}
Early advancements in scientific machine learning focused on directly embedding physical laws into neural architectures to solve boundary value problems with minimal data. Physics-Informed Neural Networks (PINNs) \cite{raissi2019physics} penalize PDE residuals at randomly sampled collocation points, while neural operators like Fourier Neural Operators \cite{li2020fourier} and DeepONet \cite{lu2021learning}, Equivariant Neural fields \cite{knigge2024spacetimecontinuouspdeforecasting} use functional mappings between infinite-dimensional spaces. To address the limitation of the availability of high-fidelity data, multi-fidelity PINNs \cite{penwarden2022multifidelity} were introduced to utilize low-fidelity responses as regularizers. However, these deterministic methods often struggle in complex settings and with real-world observations. By producing point estimates rather than predictive distributions, they offer limited uncertainty quantification and are rarely evaluated using probabilistic metrics, which can lead to physically inconsistent outputs.

\vspace{-0.8em}
\paragraph{Probabilistic Generative Modeling for Spatiotemporal Physics Systems.}
Probabilistic approaches quantify and calibrate uncertainty, providing a useful framework for learning physics-based systems.
DiffusionPDE \cite{huang2024diffusionpde} and FunDPS \cite{yao2025guided} unify the forward and backward problems through joint coefficient-solution state modeling, while VideoPDE \cite{li2025videopde} regards various tasks as video restoration to preserve fine-grained spectral details. A notable advancement is DYffusion \cite{ruhling2023dyffusion}, which replaces standard Gaussian perturbations with a dynamics-informed temporal interpolation. By avoiding the high memory overhead of video-based models like MCVD \cite{voleti2022mcvd}, DYffusion leverages Monte Carlo dropout to produce probabilistic ensembles during inference.
In physics and climate science, foundation models \cite{aich2026wind, ohana2024well, tauberschmidt2025physics} can achieve high accuracy with simple finetuning. Similarly, Rolling Sequence Diffusion Models \cite{ruhe2024rolling, wu2023ar} and ERDM \cite{cachay2025elucidated} utilize adaptive noise schedules to reflect the growth of uncertainty, prioritizing the ability to transition from deterministic to random horizons. 

\vspace{-0.8em}
\paragraph{Accelerated Inference and Consistency-Based Models.}
Diffusion’s iterative bottleneck has spurred recent studies on inference acceleration of generative models. EDM \cite{karras2022elucidating} provides efficient sampling that reduces sampling time for various tasks like molecular design \cite{vadgama2025probingequivariancesymmetrybreaking}. Rectified Flow \cite{liu2022flow} reduces transportation costs by training new ODEs on the previous flow generation pairs, optimizing the generation to a one-step path, while Shortcut Model \cite{frans2024one} stabilizes sampling through interval self-consistency. Recent innovations like MeanFlow \cite{geng2025mean} have introduced average velocity fields to characterize transitions, while Drifting Diffusion \cite{deng2026generative} performs few-step generation in feature space. Generalized flow maps \cite{davis2026generalisedflowmapsfewstep} show few-step generation on arbitrary Riemannian manifolds. Physics-informed methods like PBFM \cite{pbfm2026} further apply these ideas to physical dynamics by incorporating explicit PDE residuals into the objective. However, such methods are fundamentally constrained by their reliance on known physical formulas, making them unsuitable for complex systems where equations are unavailable or computationally prohibitive to implement. RecFM addresses these concerns by introducing a recursive framework in the data space to enforce flow trajectory across discretization scales. By adopting this approach without explicitly using PDE residual supervision, RecFM provides a robust solution for high-fidelity emulation in complex scientific domains.

\vspace{-0.3em}
\section{Experiments}\label{sec:exp}

\vspace{-0.3em}
\subsection{Datasets}
\vspace{-0.5em}
We evaluate our methods on three different dynamic physics datasets characterized by non-linear evolution and diverse spectral features. Specific technical configurations and simulation details are provided in Appendix \ref{appendix:dataset}.

\vspace{-0.8em}
\paragraph{Sea Surface Temperatures (SST).} This real-world climate dataset is adapted from the DYffusion benchmark \cite{ruhling2023dyffusion}, using daily global measurement data from the NOAA OISSTv2 \cite{huang2021improvements} product. Its spatial resolution is $1/4^\circ$. We utilized a regional $60\times 60$ latitude and longitude grid in the eastern tropical Pacific to simulate the long-term time-dependent relationship of the ocean temperature field.

\vspace{-0.8em}
\paragraph{Navier-Stokes Flow.} We follow the experimental setup of DYffusion \cite{otness2021extensible, ruhling2023dyffusion} to evaluate fluid dynamics rollouts. The environment consists of an incompressible channel flow past four randomly generated circular obstacles, inducing complex turbulence and vorticity patterns. The kinematic viscosity is set to $\nu = 10^{-3}$, and simulations are conducted on a $221 \times 42$ grid. The dataset comprises three channels: the velocity components in each spatial direction and the pressure field.

\vspace{-0.8em}
\paragraph{Helmholtz Staircase Equation.} We follow the setup of The Well \cite{ohana2024well}. This benchmark corresponds to a higher-order analytical solution for acoustic scattering from a point source near an infinite, periodic ``staircase'' boundary. The simulated fields are discretized into $1024 \times 256$ grids to capture both the real and imaginary components of the pressure field. Accordingly, the dataset consists of two channels representing the real and imaginary parts.

\vspace{-0.7em}
\subsection{Experiment Setup}
\vspace{-0.3em}

\begin{table}[!t]
    \centering
    \setlength{\tabcolsep}{4.75pt}
    \caption{Quantitative forecasting results for Sea Surface Temperature, Navier-Stokes Flow, and Helmholtz Staircase Equation. Lower values are better for MSE and CRPS, while the optimal SSR is 1. Best results in \textbf{bold}, second best \underline{underlined}, third best in \textcolor{gray}{gray}.}
    \label{tab:main_table}
    \small
    \begin{tabular}{l ccc c ccc ccc}
        \toprule
        \multirow{2}{*}{\textbf{Method}} & \multicolumn{4}{c}{\textbf{SST}} & \multicolumn{3}{c}{\textbf{Navier-Stokes}} & \multicolumn{3}{c}{\textbf{Helmholtz Staircase}}\\
        \cmidrule(lr){2-5} \cmidrule(lr){6-8} \cmidrule(lr){9-11}
        & \textbf{CRPS} & \textbf{MSE} & \textbf{SSR} & \textbf{Time [s]} & \textbf{CRPS} & \textbf{MSE} & \textbf{SSR} & \textbf{CRPS} & \textbf{MSE} & \textbf{SSR} \\
        \midrule
        Perturbation$^*$ & 0.281 & 0.180 & 0.411 & 0.4241 & 0.090 & 0.028 & 0.448 & 0.218 & 0.111 & 0.004 \\
        Dropout$^*$      & 0.267 & 0.164 & 0.406 & 0.4241 & 0.078 & 0.027 & 0.715 & 0.099 & 0.049 & 0.631 \\
        DDPM$^*$         & 0.246 & 0.177 & 0.674 & 0.3054 & 0.180 & 0.105 & 0.573 & 0.156 & 0.153 & 0.563 \\
        MCVD$^*$         & \underline{0.216} & \underline{0.161} & 0.926 & 79.167 & 0.154 & 0.070 & 0.524 & 0.137 & 0.128 & \textcolor{gray}{0.867} \\
        DYffusion$^*$    & 0.224 & 0.173 & \textcolor{gray}{1.033} & 4.6722 & 0.067 & 0.022 & 0.877 & 0.144 & 0.106 & \underline{1.121} \\
        VideoPDE \cite{li2025videopde} & \underline{0.216}  & 0.162 & 0.746  & 19.753 & \textcolor{gray}{0.033} & \textcolor{gray}{0.0068} & 0.205 & \textcolor{gray}{0.026} & \textcolor{gray}{5.6e-4} & 4.334 \\
        Vanilla FM & 0.260 & 0.232 & 0.914 & 1.5202 & 0.036 & 0.0076 & \textcolor{gray}{0.911} & 0.030 & 6.5e-4 & 1.485 \\
        \midrule
        RecFM (1-step) & \textcolor{gray}{0.217} &  \textcolor{gray}{0.162} &  \underline{0.984} & 0.4310 & \textbf{0.031} & \textbf{0.0064} & \textbf{0.959} & \underline{0.0034} & \underline{4.2e-5} & \textbf{1.090} \\
        RecFM (2-step) & \textbf{0.216} & \textbf{0.161} & \textbf{1.004} & 0.7353 &  \underline{0.032} &  \underline{0.0068} &  \underline{0.932} & \textbf{0.0027} & \textbf{2.7e-5} & 1.440 \\
        \bottomrule
        \multicolumn{11}{l}{\footnotesize $^*$Results for SST and Navier-Stokes are reproduced from DYffusion \cite{ruhling2023dyffusion}.}
    \end{tabular}
\end{table}

\subsubsection{Forecasting Configuration} 
\vspace{-0.3em}

We evaluate performance across varying temporal horizons:
\vspace{-0.7em}
\paragraph{Temporal Horizons and Autoregressive Rollout.} For SST, we predict $7$ days ahead from a $1$-day input. For Navier-Stokes and Helmholtz, we respectively perform complete trajectory reconstructions of 64 and 49 steps starting from the initial state. To manage these long-range sequences, models are applied autoregressively: Navier-Stokes models predict $16$ frames each time, while Helmholtz models generate $7$ frames, unless specified.
\vspace{-0.7em}
\paragraph{Ensemble Generation.} For all probabilistic metrics (CRPS, SSR), we create $M=50$ ensemble members per initial condition to ensure statistical reliability.
\vspace{-0.7em}
\paragraph{Model Architecture and Efficiency.} We apply RecFM to a pixel-level temporal DiT backbone, following the design introduced in \cite{li2025videopde}. All inference time measurements are performed on a single NVIDIA L40S GPU. RecFM is evaluated in both single- and multi-step regimes. More details of model architecture and implementation are included in Appendix \ref{appendix:architecture}.
\vspace{-0.7em}
\paragraph{Hyperparameter Selection.}
We use $\lambda=1$ for the consistency loss weight, with further analysis provided in Section~\ref{sec:ablation_lambda}. We adopt the depth-$2$ formulation for RecFM, corresponding to a primary trajectory with $\alpha^{(1)} = 1$ and a secondary trajectory with scale $\alpha^{(2)} = \alpha$, as it provides the best performance and efficiency (see Appendix~\ref{appendix:depth}).

\vspace{-0.3em}
\subsubsection{Baselines} 
\vspace{-0.3em}

We compare RecFM against a comprehensive suite of generative and stochastic benchmarks. For standard forecasting models, we adopt the experimental configuration and model suite from DYffusion \cite{ruhling2023dyffusion}, which includes:
\begin{itemize}
    \vspace{-0.4em}
    \item \textbf{Stochastic Methods:} \textbf{Perturbation} and \textbf{Dropout} (Monte Carlo dropout at inference).
    \item \textbf{Iterative Models:} \textbf{DDPM} and \textbf{MCVD}, which utilize Gaussian noising processes.
    \item \textbf{Dynamics-Informed Solvers:} \textbf{DYffusion}, which directly couples diffusion steps with physical timesteps.
    \vspace{-0.4em}
\end{itemize}
We further include benchmarks with state-of-the-art generative backbones:
\begin{itemize}
    \vspace{-0.4em}
    \item \textbf{VideoPDE \cite{li2025videopde}:} A unified solver that recasts PDE solving as hierarchical video inpainting using a pixel-space hierarchical transformer. 
    \item \textbf{Vanilla FM:} Utilizes the identical architectural backbone to RecFM but is trained using a standard Flow Matching objective without the recursive feature.
    \vspace{-0.4em}
\end{itemize}
We exclude PBFM \cite{pbfm2026} from our primary comparisons as it requires explicit closed-form PDE residuals, which are impractical for complex, data-rich systems like global SST measurements. Comparison of PDE-governed data with physics-informed metrics is included in Appendix \ref{appendix:pbfm}. One-step methods such as MeanFlow \cite{geng2025mean} and Shortcut Models \cite{frans2024one} and other benchmarks like Rectified Flow \cite{liu2022flow} are primarily designed for static generation and are therefore not included in our main comparisons. A comparison with Shortcut Models on Helmholtz Staircase is included in Appendix~\ref{app:shortcut-recfm}.

\subsubsection{Evaluation Metrics}
\vspace{-0.3em}

To evaluate the fidelity and calibration of probability prediction, we employ three standard metrics:

\vspace{-0.5em}
\paragraph{Continuous Ranked Probability Score (CRPS) \cite{matheson1976scoring}.} A strictly proper scoring rule to measure the accuracy of the cumulative distribution function $F$ relative to the observation $y$: \vspace{-0.2em}
    \begin{equation}
        \text{CRPS}(F, y) = \int_{-\infty}^{\infty} (F(z) - \mathbf{1}[z \geq y])^2 dz
    \end{equation}
    In practice, we use the unbiased ``fair'' estimator for $M$ ensembles.

\vspace{-0.5em}
\paragraph{Mean Squared Error (MSE).} Measures the deterministic accuracy of the ensemble mean prediction $\bar{x}$ against the ground truth $y$ for the dataset of size $S$: 
    \begin{equation}
        \text{MSE} = \frac{1}{S} \sum_{j=1}^S \|\bar{x}_j - y_j\|^2
    \end{equation}
\paragraph{Spread-Skill Ratio (SSR).} 
\vspace{-0.5em}Evaluates the reliability of the ensemble by comparing the ensemble spread to the RMSE of the ensemble mean. An ideal ratio of $1.0$ indicates a perfectly calibrated ensemble. Specifically, SSR values smaller than 1.0 indicate underdispersion, while values larger than 1.0 indicate overdispersion.

\vspace{-0.5em}

\subsection{Forecasting Results}
\vspace{-0.3em}

\begin{figure}
    \centering
    \includegraphics[width=0.999\linewidth]{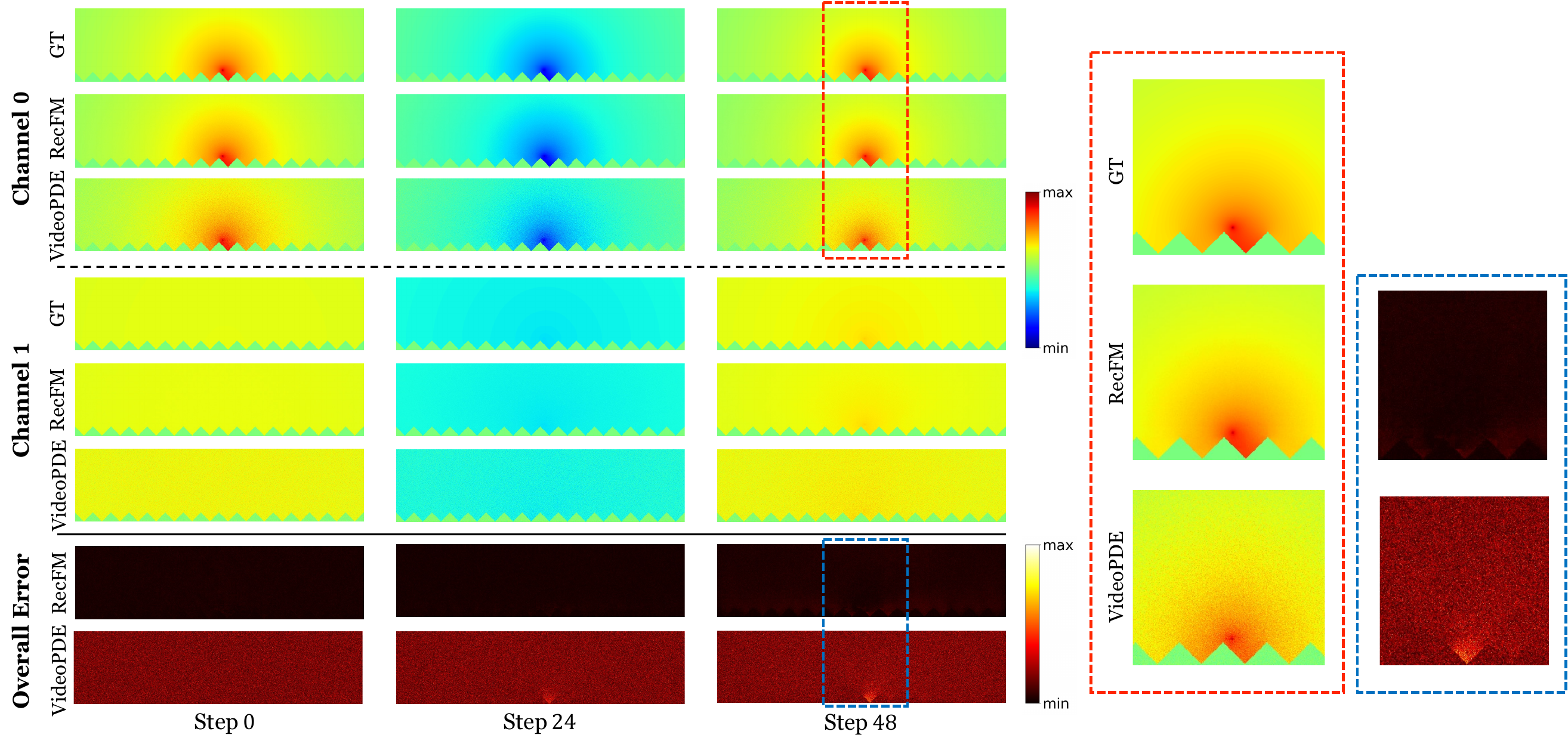}
    \caption{\textbf{Roll-out results of the Helmholtz Staircase equation.} Visual comparison of Ground Truth against RecFM and VideoPDE (best-performed baseline) for two channels, with the bottom rows indicating absolute errors. Columns correspond to dataset timesteps. The variation observed at Step 48 is displayed in an enlarged view on the right.}
    \label{fig:helmholtz_zoom}
\end{figure}

Quantitative results across all benchmarks are summarized in Table~\ref{tab:main_table}, with standard deviations in Appendix~\ref{appendix:stdev}. We evaluate RecFM using one- and two-step inference to highlight its flexibility, with additional analysis of step size in Appendix~\ref{appendix:mse_vs_step}. 
Overall, RecFM consistently achieves state-of-the-art performance in both fidelity and efficiency. In particular, it attains up to a $20\times$ speedup over the diffusion-based baseline VideoPDE while also improving predictive accuracy and calibration. 
This speedup is measured in terms of total rollout runtime, reflecting the reduced number of inference steps required by RecFM.
On the Helmholtz Staircase equation, RecFM achieves a $10\times$ reduction in error compared to VideoPDE, which is the best-performed baseline. 
We further visualize roll-out snapshots for both channels of the Helmholtz equation, along with corresponding error maps, in Figure~\ref{fig:helmholtz_zoom}. RecFM produces predictions that closely match the ground truth, while VideoPDE struggles to capture the circular wave propagation patterns. Additional visualizations are provided in Appendix~\ref{appendix:vis}.

Moreover, compared to vanilla flow-matching methods, which typically require $\sim$5 inference steps, RecFM produces high-quality results with only 1-2 steps, achieving over $15\%$ lower MSE and substantially better SSR scores. We also observe that multi-step RecFM models do not consistently outperform single-step variants, as errors can accumulate over successive iterations.

While RecFM performs best across all tasks, its advantage is more pronounced on deterministic problems, such as PDE prediction tasks governed by explicit physical constraints, than on more stochastic data such as SST. This behavior is expected, as few-step flow matching is inherently more deterministic, which introduces minimal randomness into the sampling process.

\vspace{-0.8em}
\paragraph{Training Stability.}
\begin{wrapfigure}{r}{0.5\linewidth}
    \centering
    \includegraphics[width=\linewidth]{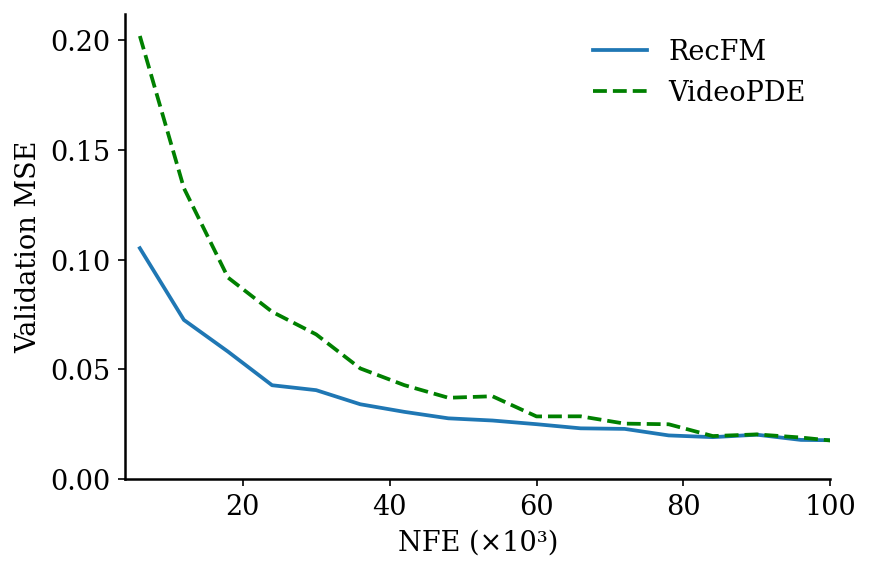}
    \caption{\textbf{Validation MSE versus NFE during training.} RecFM converges faster than the diffusion-based model VideoPDE and maintains consistently lower validation error.}
    \label{fig:mse_vs_nfe_videopde}
    \vspace{-2em}
\end{wrapfigure}

We measure training progress of Navier-Stokes Flow in terms of the number of function evaluations (NFE), defined as the total number of vector field evaluations (\textit{i.e.}, forward passes) during optimization. As shown in Figure~\ref{fig:mse_vs_nfe_videopde}, RecFM converges faster than the diffusion-based baseline (VideoPDE) and consistently achieves lower validation error throughout training.

\vspace{-0.5em}
\subsection{Ablation Studies}\label{sec:ablation_lambda}
\vspace{-0.3em}

We investigate the sensitivity of RecFM (see Equation \ref{eq:training_objective}) to the consistency loss weight $\lambda$ on the Navier-Stokes equation. Table \ref{tab:ablation_lambda} reports results for both single-step and multi-step models over a wide range of $\lambda$ values. A moderate setting (\textit{e.g.}, $\lambda = 1.0$) consistently yields the best performance, suggesting that an appropriate balance between objectives is important. When $\lambda$ is too small, the model places insufficient emphasis on trajectory consistency, which degrades performance. Interestingly, even without the self-consistency term ($\lambda = 0$), RecFM still surpasses the Vanilla FM baseline in Table \ref{tab:main_table}, likely due to the presence of the secondary trajectory loss. In contrast, very large values (\textit{e.g.}, $\lambda = 10^6$) cause the consistency term to overwhelm the primary flow-matching objective, leading to a marked drop in accuracy.

\vspace{-1em}

\begin{table}[!ht]
    \centering
    \caption{Ablation study on the effect of $\lambda$ for the Navier-Stokes equation using the 1-step model and 5-step model. Here $\lambda=0$ results in vanilla FM. Best results in \textbf{bold}.}
    \label{tab:ablation_lambda}
    \small
    \begin{tabular}{lcccccc}
        \toprule
        & \multicolumn{3}{c}{\textbf{1-step model}} & \multicolumn{3}{c}{\textbf{5-step model}} \\
        \cmidrule(lr){2-4} \cmidrule(lr){5-7}
        $\lambda$ & \textbf{CRPS} ($\downarrow$) & \textbf{MSE} ($\downarrow$) & \textbf{SSR} ($\rightarrow 1$) & \textbf{CRPS} ($\downarrow$) & \textbf{MSE} ($\downarrow$) & \textbf{SSR} ($\rightarrow 1$) \\
        \midrule
        0.0   & 0.035 & 0.0074 & 0.957 & 0.039 & 0.0093 & 0.843 \\
        0.5   & 0.034 & 0.0071 & 1.024 & 0.040 & 0.0099 & \textbf{0.855}  \\
        1.0   & \textbf{0.031} & \textbf{0.0064} & 0.959 & \textbf{0.037} & \textbf{0.0083} & 0.836 \\
        10.0  & 0.038 & 0.0089 & \textbf{0.988} & 0.039 & 0.0089 & 0.784 \\
        1000000 & 0.238 & 0.268 & 1.147 & 0.234 & 0.261 & 1.161 \\
        \bottomrule
    \end{tabular}
\end{table}
\vspace{-0.5em}

\subsection{Image Generation Experiments}
\vspace{-0.3em}
Although our primary focus is scientific emulation, RecFM also generalizes beyond physics-based systems, achieving competitive image quality with reduced training and inference cost (Appendix \ref{appendix:img}).
\vspace{-0.5em}

\section{Conclusion and Discussion}\label{sec:conclusion}
\vspace{-0.5em}

We introduced Recursive Flow Matching (RecFM), a framework that enforces consistency of generative trajectories across different sampling regimes. Our findings indicate that aligning trajectories across scales leads to an unexpected effect: using fewer inference steps can actually enhance both stability and accuracy, particularly in physics-based applications. This observation questions the usual trade-off between sampling efficiency and fidelity and suggests that the structure of the trajectory, rather than simply increasing the number of steps, is central to effective generation.

In experiments, RecFM delivers strong results on a range of scientific benchmarks, matching the performance of multi-step solvers while operating in the one- or few-step regimes and yielding notable speedups. These results highlight the potential of consistency-based approaches for real-time scientific emulation.

\vspace{-0.5em}
\paragraph{Limitations and Future Work.}
Despite these gains, extending RecFM to high-complexity real-world video remains challenging. Unlike physics-driven systems, natural videos involve rich semantic and temporal variations that may require modeling beyond standard flow matching trajectories. Although preliminary image-generation results are promising, scaling the framework to realistic video domains remains an open problem. Future work will further explore these settings and investigate RecFM as a general-purpose foundation model for multi-physics and real-world dynamical systems.

\newpage

\bibliographystyle{unsrt}
\bibliography{references}

@inproceedings{nichol2021improveddenoisingdiffusionprobabilistic,
  title={Improved denoising diffusion probabilistic models},
  author={Nichol, Alexander Quinn and Dhariwal, Prafulla},
  booktitle={International conference on machine learning},
  pages={8162--8171},
  year={2021},
  organization={PMLR}
}

@article{chen2024flowmatchinggeneralgeometries,
  title={Flow matching on general geometries},
  author={Chen, Ricky TQ and Lipman, Yaron},
  journal={arXiv preprint arXiv:2302.03660},
  year={2023}
}

@article{song2020denoising,
  title={Denoising diffusion implicit models},
  author={Song, Jiaming and Meng, Chenlin and Ermon, Stefano},
  journal={arXiv preprint arXiv:2010.02502},
  year={2020}
}

@article{ho2020denoising,
  title={Denoising diffusion probabilistic models},
  author={Ho, Jonathan and Jain, Ajay and Abbeel, Pieter},
  journal={Advances in neural information processing systems},
  volume={33},
  pages={6840--6851},
  year={2020}
}

@article{karras2022elucidating,
  title={Elucidating the design space of diffusion-based generative models},
  author={Karras, Tero and Aittala, Miika and Aila, Timo and Laine, Samuli},
  journal={Advances in neural information processing systems},
  volume={35},
  pages={26565--26577},
  year={2022}
}

@article{xu2023inversion,
  title={Inversion-free image editing with natural language},
  author={Xu, Sihan and Huang, Yidong and Pan, Jiayi and Ma, Ziqiao and Chai, Joyce},
  journal={arXiv preprint arXiv:2312.04965},
  year={2023}
}

@article{vadgama2025probingequivariancesymmetrybreaking,
  title={Probing Equivariance and Symmetry Breaking in Convolutional Networks},
  author={Vadgama, Sharvaree and Islam, Mohammad Mohaiminul and Buracas, Domas and Shewmake, Christian and Moskalev, Artem and Bekkers, Erik},
  journal={arXiv preprint arXiv:2501.01999},
  year={2025}
}

@article{davis2026generalisedflowmapsfewstep,
  title={Generalised Flow Maps for Few-Step Generative Modelling on Riemannian Manifolds},
  author={Davis, Oscar and Albergo, Michael S and Boffi, Nicholas M and Bronstein, Michael M and Bose, Avishek Joey},
  journal={arXiv preprint arXiv:2510.21608},
  year={2025}
}

@article{knigge2024spacetimecontinuouspdeforecasting,
  title={Space-time continuous pde forecasting using equivariant neural fields},
  author={Knigge, David M and Wessels, David R and Valperga, Riccardo and Papa, Samuele and Sonke, Jan-Jakob and Gavves, Efstratios and Bekkers, Erik J},
  journal={Advances in Neural Information Processing Systems},
  volume={37},
  pages={76553--76577},
  year={2024}
}

@article{benton2024error,
  title={Error Bounds for Flow Matching Methods},
  author={Benton, Joe and Deligiannidis, George and Doucet, Arnaud},
  journal={Transactions on Machine Learning Research},
  year={2024}
}

@article{zhang2025trajectoryflowmatchingapplications,
  title={Trajectory flow matching with applications to clinical time series modelling},
  author={Zhang, Xi and Pu, Yuan and Kawamura, Yuki and Loza, Andrew and Bengio, Yoshua and Shung, Dennis L and Tong, Alexander},
  journal={Advances in Neural Information Processing Systems},
  volume={37},
  pages={107198--107224},
  year={2024}
}

@article{islam2025longitudinalflowmatchingtrajectory,
  title={Longitudinal Flow Matching for Trajectory Modeling},
  author={Islam, Mohammad Mohaiminul and Kuipers, Thijs P and Vadgama, Sharvaree and de Vente, Coen and Khan, Afsana and S{\'a}nchez, Clara I and Bekkers, Erik J},
  journal={arXiv preprint arXiv:2510.03569},
  year={2025}
}

@inproceedings{shen2025simultaneous,
  title={Simultaneous modeling of protein conformation and dynamics via autoregression},
  author={Shen, Yuning and Wang, Lihao and Yuan, Huizhuo and Wang, Yan and Yang, Bangji and Gu, Quanquan},
  booktitle={The Thirty-ninth Annual Conference on Neural Information Processing Systems},
  year={2025}
}

@article{abramson2024accurate,
  title={Accurate structure prediction of biomolecular interactions with AlphaFold 3},
  author={Abramson, Josh and Adler, Jonas and Dunger, Jack and Evans, Richard and Green, Tim and Pritzel, Alexander and Ronneberger, Olaf and Willmore, Lindsay and Ballard, Andrew J and Bambrick, Joshua and others},
  journal={Nature},
  volume={630},
  number={8016},
  pages={493--500},
  year={2024},
  publisher={Nature Publishing Group UK London}
}

@article{watt2025ace2,
  title={ACE2: accurately learning subseasonal to decadal atmospheric variability and forced responses},
  author={Watt-Meyer, Oliver and Henn, Brian and McGibbon, Jeremy and Clark, Spencer K and Kwa, Anna and Perkins, W Andre and Wu, Elynn and Harris, Lucas and Bretherton, Christopher S},
  journal={npj Climate and Atmospheric Science},
  volume={8},
  number={1},
  pages={205},
  year={2025},
  publisher={Nature Publishing Group UK London}
}

@article{duncan2025samudrace,
  title={SamudrACE: Fast and accurate coupled climate modeling with 3D ocean and atmosphere emulators},
  author={Duncan, James PC and Wu, Elynn and Dheeshjith, Surya and Subel, Adam and Arcomano, Troy and Clark, Spencer K and Henn, Brian and Kwa, Anna and McGibbon, Jeremy and Perkins, W Andre and others},
  journal={arXiv preprint arXiv:2509.12490},
  year={2025}
}

@article{ruhling2023dyffusion,
  title={Dyffusion: A dynamics-informed diffusion model for spatiotemporal forecasting},
  author={R{\"u}hling Cachay, Salva and Zhao, Bo and Joren, Hailey and Yu, Rose},
  journal={Advances in neural information processing systems},
  volume={36},
  pages={45259--45287},
  year={2023}
}

@article{cachay2025elucidated,
  title={Elucidated Rolling Diffusion Models for Probabilistic Forecasting of Complex Dynamics},
  author={Cachay, Salva R{\"u}hling and Aittala, Miika and Kreis, Karsten and Brenowitz, Noah and Vahdat, Arash and Mardani, Morteza and Yu, Rose},
  journal={arXiv preprint arXiv:2506.20024},
  year={2025}
}

@article{huang2024diffusionpde,
  title={DiffusionPDE: Generative PDE-solving under partial observation},
  author={Huang, Jiahe and Yang, Guandao and Wang, Zichen and Park, Jeong Joon},
  journal={Advances in Neural Information Processing Systems},
  volume={37},
  pages={130291--130323},
  year={2024}
}

@article{zhuang2025spatially,
  title={Spatially-aware diffusion models with cross-attention for global field reconstruction with sparse observations},
  author={Zhuang, Yilin and Cheng, Sibo and Duraisamy, Karthik},
  journal={Computer Methods in Applied Mechanics and Engineering},
  volume={435},
  pages={117623},
  year={2025},
  publisher={Elsevier}
}

@article{frans2024one,
  title={One step diffusion via shortcut models},
  author={Frans, Kevin and Hafner, Danijar and Levine, Sergey and Abbeel, Pieter},
  journal={arXiv preprint arXiv:2410.12557},
  year={2024}
}

@article{mathieu2020riemannian,
  title={Riemannian continuous normalizing flows},
  author={Mathieu, Emile and Nickel, Maximilian},
  journal={Advances in neural information processing systems},
  volume={33},
  pages={2503--2515},
  year={2020}
}

@article{lipman2022flow,
  title={Flow matching for generative modeling},
  author={Lipman, Yaron and Chen, Ricky TQ and Ben-Hamu, Heli and Nickel, Maximilian and Le, Matt},
  journal={arXiv preprint arXiv:2210.02747},
  year={2022}
}

@article{geng2025mean,
  title={Mean flows for one-step generative modeling},
  author={Geng, Zhengyang and Deng, Mingyang and Bai, Xingjian and Kolter, J Zico and He, Kaiming},
  journal={arXiv preprint arXiv:2505.13447},
  year={2025}
}

@article{li2025videopde,
  title={VideoPDE: Unified generative pde solving via video inpainting diffusion models},
  author={Li, Edward and Wang, Zichen and Huang, Jiahe and Park, Jeong Joon},
  journal={arXiv preprint arXiv:2506.13754},
  year={2025}
}

@article{raissi2019physics,
  title={Physics-informed neural networks: A deep learning framework for solving forward and inverse problems involving nonlinear partial differential equations},
  author={Raissi, Maziar and Perdikaris, Paris and Karniadakis, George E},
  journal={Journal of Computational physics},
  volume={378},
  pages={686--707},
  year={2019},
  publisher={Elsevier}
}

@article{li2020fourier,
  title={Fourier neural operator for parametric partial differential equations},
  author={Li, Zongyi and Kovachki, Nikola and Azizzadenesheli, Kamyar and Liu, Burigede and Bhattacharya, Kaushik and Stuart, Andrew and Anandkumar, Anima},
  journal={arXiv preprint arXiv:2010.08895},
  year={2020}
}

@article{lu2021learning,
  title={Learning nonlinear operators via DeepONet based on the universal approximation theorem of operators},
  author={Lu, Lu and Jin, Pengzhan and Pang, Guofei and Zhang, Zhongqiang and Karniadakis, George Em},
  journal={Nature machine intelligence},
  volume={3},
  number={3},
  pages={218--229},
  year={2021},
  publisher={Nature Publishing Group UK London}
}

@article{penwarden2022multifidelity,
  title={Multifidelity modeling for physics-informed neural networks (PINNs)},
  author={Penwarden, Michael and Zhe, Shandian and Narayan, Akil and Kirby, Robert M},
  journal={Journal of Computational Physics},
  volume={451},
  pages={110844},
  year={2022},
  publisher={Elsevier}
}

@article{yao2025guided,
  title={Guided diffusion sampling on function spaces with applications to pdes},
  author={Yao, Jiachen and Mammadov, Abbas and Berner, Julius and Kerrigan, Gavin and Ye, Jong Chul and Azizzadenesheli, Kamyar and Anandkumar, Anima},
  journal={arXiv preprint arXiv:2505.17004},
  year={2025}
}

@article{aich2026wind,
  title={WIND: Weather Inverse Diffusion for Zero-Shot Atmospheric Modeling},
  author={Aich, Michael and F{\"u}rst, Andreas and Sestak, Florian and Ruiz-Gonzalez, Carlos and Boers, Niklas and Brandstetter, Johannes},
  journal={arXiv preprint arXiv:2602.03924},
  year={2026}
}

@article{ohana2024well,
  title={The well: a large-scale collection of diverse physics simulations for machine learning},
  author={Ohana, Ruben and McCabe, Michael and Meyer, Lucas and Morel, Rudy and Agocs, Fruzsina J and Beneitez, Miguel and Berger, Marsha and Burkhart, Blakesley and Dalziel, Stuart B and Fielding, Drummond B and others},
  journal={Advances in Neural Information Processing Systems},
  volume={37},
  pages={44989--45037},
  year={2024}
}

@article{ruhe2024rolling,
  title={Rolling diffusion models},
  author={Ruhe, David and Heek, Jonathan and Salimans, Tim and Hoogeboom, Emiel},
  journal={arXiv preprint arXiv:2402.09470},
  year={2024}
}

@article{wu2023ar,
  title={Ar-diffusion: Auto-regressive diffusion model for text generation},
  author={Wu, Tong and Fan, Zhihao and Liu, Xiao and Zheng, Hai-Tao and Gong, Yeyun and Jiao, Jian and Li, Juntao and Guo, Jian and Duan, Nan and Chen, Weizhu and others},
  journal={Advances in Neural Information Processing Systems},
  volume={36},
  pages={39957--39974},
  year={2023}
}

@article{voleti2022mcvd,
  title={Mcvd-masked conditional video diffusion for prediction, generation, and interpolation},
  author={Voleti, Vikram and Jolicoeur-Martineau, Alexia and Pal, Chris},
  journal={Advances in neural information processing systems},
  volume={35},
  pages={23371--23385},
  year={2022}
}

@article{liu2022flow,
  title={Flow straight and fast: Learning to generate and transfer data with rectified flow},
  author={Liu, Xingchao and Gong, Chengyue and Liu, Qiang},
  journal={arXiv preprint arXiv:2209.03003},
  year={2022}
}

@article{deng2026generative,
  title={Generative Modeling via Drifting},
  author={Deng, Mingyang and Li, He and Li, Tianhong and Du, Yilun and He, Kaiming},
  journal={arXiv preprint arXiv:2602.04770},
  year={2026}
}

@inproceedings{pbfm2026,
  title={Physics vs distributions: Pareto optimal flow matching with physics constraints},
  author={Baldan, Giacomo and Liu, Qiang and Guardone, Alberto and Thuerey, Nils},
  booktitle={The Fourteenth International Conference on Learning Representations},
  year={2026}
}

@article{huang2021improvements,
  title={Improvements of the daily optimum interpolation sea surface temperature (DOISST) version 2.1},
  author={Huang, Boyin and Liu, Chunying and Banzon, Viva and Freeman, Eric and Graham, Garrett and Hankins, Bill and Smith, Tom and Zhang, Huai-Min},
  journal={Journal of Climate},
  volume={34},
  number={8},
  pages={2923--2939},
  year={2021},
  publisher={American Meteorological Society}
}

@article{otness2021extensible,
  title={An extensible benchmark suite for learning to simulate physical systems},
  author={Otness, Karl and Gjoka, Arvi and Bruna, Joan and Panozzo, Daniele and Peherstorfer, Benjamin and Schneider, Teseo and Zorin, Denis},
  journal={arXiv preprint arXiv:2108.07799},
  year={2021}
}

@article{matheson1976scoring,
  title={Scoring rules for continuous probability distributions},
  author={Matheson, James E and Winkler, Robert L},
  journal={Management science},
  volume={22},
  number={10},
  pages={1087--1096},
  year={1976},
  publisher={INFORMS}
}

@inproceedings{ma2024sit,
  title={Sit: Exploring flow and diffusion-based generative models with scalable interpolant transformers},
  author={Ma, Nanye and Goldstein, Mark and Albergo, Michael S and Boffi, Nicholas M and Vanden-Eijnden, Eric and Xie, Saining},
  booktitle={European Conference on Computer Vision},
  pages={23--40},
  year={2024},
  organization={Springer}
}

@inproceedings{peebles2023scalable,
  title={Scalable diffusion models with transformers},
  author={Peebles, William and Xie, Saining},
  booktitle={Proceedings of the IEEE/CVF international conference on computer vision},
  pages={4195--4205},
  year={2023}
}

@article{dhariwal2021diffusion,
  title={Diffusion models beat gans on image synthesis},
  author={Dhariwal, Prafulla and Nichol, Alexander},
  journal={Advances in neural information processing systems},
  volume={34},
  pages={8780--8794},
  year={2021}
}

@inproceedings{rombach2022high,
  title={High-resolution image synthesis with latent diffusion models},
  author={Rombach, Robin and Blattmann, Andreas and Lorenz, Dominik and Esser, Patrick and Ommer, Bj{\"o}rn},
  booktitle={Proceedings of the IEEE/CVF conference on computer vision and pattern recognition},
  pages={10684--10695},
  year={2022}
}

@article{zeni2025generative,
  title={A generative model for inorganic materials design},
  author={Zeni, Claudio and Pinsler, Robert and Z{\"u}gner, Daniel and Fowler, Andrew and Horton, Matthew and Fu, Xiang and Wang, Zilong and Shysheya, Aliaksandra and Crabb{\'e}, Jonathan and Ueda, Shoko and others},
  journal={Nature},
  volume={639},
  number={8055},
  pages={624--632},
  year={2025},
  publisher={Nature Publishing Group UK London}
}

@inproceedings{deng2009imagenet,
  title={Imagenet: A large-scale hierarchical image database},
  author={Deng, Jia and Dong, Wei and Socher, Richard and Li, Li-Jia and Li, Kai and Fei-Fei, Li},
  booktitle={2009 IEEE conference on computer vision and pattern recognition},
  pages={248--255},
  year={2009},
  organization={Ieee}
}

@article{tauberschmidt2025physics,
  title={Physics-Constrained Fine-Tuning of Flow-Matching Models for Generation and Inverse Problems},
  author={Tauberschmidt, Jan and Fellenz, Sophie and Vollmer, Sebastian J and Duncan, Andrew B},
  journal={arXiv preprint arXiv:2508.09156},
  year={2025}
}

@book{dhatt2012finite,
  title={Finite element method},
  author={Dhatt, Gouri and Lefran{\c{c}}ois, Emmanuel and Touzot, Gilbert},
  year={2012},
  publisher={John Wiley \& Sons}
}

@article{xu2025understanding,
  title={On understanding and overcoming spectral biases of deep neural network learning methods for solving PDEs},
  author={Xu, Zhi-Qin John and Zhang, Lulu and Cai, Wei},
  journal={Journal of Computational Physics},
  volume={530},
  pages={113905},
  year={2025},
  publisher={Elsevier}
}

@article{cantwell2015nektar++,
  title={Nektar++: An open-source spectral/hp element framework},
  author={Cantwell, Chris D and Moxey, David and Comerford, Andrew and Bolis, Alessandro and Rocco, Gabriele and Mengaldo, Gianmarco and De Grazia, Daniele and Yakovlev, Sergey and Lombard, J-E and Ekelschot, Dirk and others},
  journal={Computer physics communications},
  volume={192},
  pages={205--219},
  year={2015},
  publisher={Elsevier}
}

@article{kovachki2023neural,
  title={Neural operator: Learning maps between function spaces with applications to pdes},
  author={Kovachki, Nikola and Li, Zongyi and Liu, Burigede and Azizzadenesheli, Kamyar and Bhattacharya, Kaushik and Stuart, Andrew and Anandkumar, Anima},
  journal={Journal of Machine Learning Research},
  volume={24},
  number={89},
  pages={1--97},
  year={2023}
}

@article{xu2023cyclenet,
  title={Cyclenet: Rethinking cycle consistency in text-guided diffusion for image manipulation},
  author={Xu, Sihan and Ma, Ziqiao and Huang, Yidong and Lee, Honglak and Chai, Joyce},
  journal={Advances in Neural Information Processing Systems},
  volume={36},
  pages={10359--10384},
  year={2023}
}

@article{song2024multi,
  title={Multi-student diffusion distillation for better one-step generators},
  author={Song, Yanke and Lorraine, Jonathan and Nie, Weili and Kreis, Karsten and Lucas, James},
  journal={arXiv preprint arXiv:2410.23274},
  year={2024}
}


\newpage

\appendix

\section{Dataset Details}\label{appendix:dataset}

In this section, we provide the formal governing equations and technical implementation details for the physics datasets used in our evaluation. For every dataset that includes boundary conditions, these conditions are provided as extra constraints to each model.

\subsection{Sea Surface Temperatures (SST)}

The SST dataset is a real-world climate benchmark representing the daily evolution of sea surface temperature fields $\mathbf{T}$ over the eastern tropical Pacific Ocean. While not governed by a single closed-form PDE, the dynamics arise from complex ocean-atmosphere interactions and large-scale climate variability. Each sample consists of 11 spatial boxes, each represented as a $60 \times 60$ latitude-longitude grid with a single scalar channel corresponding to temperature values \cite{ruhling2023dyffusion}. Following prior work, we use data from 1982 to 2019 for training, 2020 for validation, and 2021 for testing.

\subsection{Navier-Stokes (NS) Flow}

The Navier-Stokes benchmark simulates incompressible channel flow past random circular obstacles. The dynamics are governed by the following momentum and continuity equations:
\begin{equation}
    \begin{aligned}
        \frac{\partial \mathbf{u}}{\partial t} + (\mathbf{u} \cdot \nabla) \mathbf{u} &= -\frac{1}{\rho} \nabla p + \nu \nabla^2 \mathbf{u} + \mathbf{f}\\
        \nabla \cdot \mathbf{u} &= 0
    \end{aligned}
\end{equation}
where $\mathbf{u} = (u, v)$ is the velocity vector, $p$ is the pressure field, $\rho$ is the fluid density, and $\nu = 10^{-3}$ is the kinematic viscosity.

\paragraph{Technical Configuration:}
\begin{itemize}
    \item \textbf{Channels:} The dataset consists of 3 distinct channels: the $x$-velocity component ($u$), the $y$-velocity component ($v$), and the pressure field ($p$).
    \item \textbf{Preprocessing:} The raw simulation data is defined on a $221 \times 42$ grid. For the implementation of RecFM, Vanilla FM, and VideoPDE \cite{li2025videopde}, the input fields are bilinearly interpolated to a resolution of $220 \times 40$ before being passed to the model. During evaluation, the generated outputs are upsampled back to the original $221 \times 42$ resolution to compute metrics.
\end{itemize}

\subsection{Helmholtz Staircase Equation}

The Helmholtz benchmark evaluates acoustic scattering from a point source near a corrugated boundary. The steady-state pressure field $u$ satisfies:
\begin{equation}
    -(\Delta + \omega^2)u = \delta_{\mathbf{x}_0}
\end{equation}
where $\Delta$ is the Laplacian, $\omega$ is the angular frequency, and $\mathbf{x}_0$ is the source position. The time-dependent pressure evolution is defined analytically as:
\begin{equation}
    U(t, \mathbf{x}) = u(\mathbf{x}) e^{-i \omega t}
\end{equation}
Although the time dependence is analytically separable and purely periodic, such that the full trajectory is determined by the spatial field at a single time, the dataset remains physically meaningful as it encodes coherent wave propagation and phase dynamics. This setting therefore serves as a controlled benchmark for evaluating physical consistency (see Appendix~\ref{appendix:pbfm}), complementing more dynamically complex systems.

\paragraph{Technical Configuration:}
\begin{itemize}
    \item \textbf{Channels:} To represent the complex-valued pressure fields, the model processes $2$ primary channels: the real component $Re(U)$ and the imaginary component $Im(U)$ of the acoustic pressure. For all models, the constant domain masks provided in the dataset \cite{ohana2024well} are used as a third input channel.
\end{itemize}

\section{Additional Theorems and Corollaries}\label{appendix:proof}

\begin{proposition}[Trajectory Convergence]
\label{prop:convergence}
Let $\theta^*$ be a global minimizer of $\mathcal{L}_{\textup{total}}$
over a sufficiently expressive function class. Let $x_t = (1-t)\,x_0 + t\,x_1$ with $x_0 \sim p_0$, $x_1 \sim p_1$, and $v^* = x_1 - x_0$. Then the global minimizer of $\mathcal{L}_{\textup{pri}}$ recovers the conditional expectation
\begin{equation}\label{eq:vfm}
    v_{\theta^*}(x, t, 1) \;=\; \mathbb{E}\!\left[\, x_1 - x_0 \;\middle|\; x_t = x \,\right],
\end{equation}
and generates the correct marginal path $p_t$ for all $t \in [0,1]$~\cite{lipman2022flow}.
Jointly, the global minimizer of $\mathcal{L}_{\textup{sec}}$ satisfies, for every $\alpha \in (0,1]$ and $\tau = t/\alpha$,
\begin{equation}\label{eq:vsec_opt}
    v_{\theta^*}(x,\,\tau,\,\alpha) \;=\; \alpha\, v_{\theta^*}(x,\,t,\,1),
\end{equation}
with $\mathcal{L}_{\textup{cons}} = 0$ holding automatically at this optimum.
\end{proposition}
 
\begin{proof}
The loss $\mathcal{L}_{\textup{pri}} = \mathbb{E}_{t,x_0,x_1}[\| v_\theta(x_t, t, 1) - v^* \|^2]$
is a conditional regression whose unique $L^2$ minimizer is the conditional expectation~\eqref{eq:vfm}.
By the theory of Continuous Normalizing Flows~\cite{lipman2022flow}, the ODE $\frac{\mathrm{d}\psi_t}{\mathrm{d}t} = v_{\theta^*}(\psi_t, t, 1)$ with $\psi_0 \sim p_0$ generates the correct marginal $p_t$.
The secondary loss $\mathcal{L}_{\textup{sec}} = \| \hat{v}_{\textup{sec}} - \alpha\, v^* \|^2$ is minimized at
\[
    v_{\theta^*}(x_t,\, \tau,\, \alpha) \;=\; \alpha\,\mathbb{E}\!\left[ x_1 - x_0 \;\middle|\; x_t = x \right] \;=\; \alpha\, v_{\theta^*}(x_t,\, t,\, 1),
    \quad \tau = t/\alpha,
\]
so that $\mathcal{L}_{\textup{cons}} = \| \hat{v}_{\textup{sec}} - \alpha\,\hat{v}_{\textup{pri}} \|^2 = 0$ by construction.
\end{proof}

\begin{theorem}[Marginal Preservation of the Secondary Trajectory]
\label{thm:marginal}
Let $v_{\theta^*}$ be as in Proposition~\ref{prop:convergence}, and assume that $v_{\theta^*}(\cdot, t, \alpha)$ is Lipschitz in $x$ uniformly over $t$ and $\alpha$. For a fixed $\alpha \in (0,1]$, consider the secondary trajectory ODE:
\begin{equation}
    \frac{dx}{d\tau} = v_{\theta^*}(x,\,\tau,\,\alpha),
    \qquad x(0) \sim p_0.
    \label{eq:secondary_ode}
\end{equation}
Let $\{q_\tau^{(\alpha)}\}_{\tau \in [0,1]}$ denote the probability
path induced by this ODE. Then for all $\tau \in [0, 1]$,
$q_\tau^{(\alpha)}$ coincides with the marginal distribution of
$(1 - \alpha\tau)\,x_0 + \alpha\tau\, x_1$, where $x_0 \sim p_0$
and $x_1 \sim p_1$. In particular,
\begin{equation}\label{eq:marginal_form}
    q_0^{(\alpha)} = p_0, \qquad q_1^{(\alpha)} = p_\alpha,
\end{equation}
where $p_\alpha$ is the marginal distribution of
$(1-\alpha)\,x_0 + \alpha\,x_1$.
\end{theorem}
 
\begin{proof}
Define the candidate path $\tilde{x}_\tau := (1-\alpha\tau)\,x_0 + \alpha\tau\, x_1$.
Differentiating: $\frac{\mathrm{d}\tilde{x}_\tau}{\mathrm{d}\tau} = \alpha\,(x_1 - x_0)$.
Its marginal velocity field is
\[
    u_\tau^{(\alpha)}(x) := \mathbb{E}\!\left[\frac{\mathrm{d}\tilde{x}_\tau}{\mathrm{d}\tau} \;\middle|\; \tilde{x}_\tau = x\right] = \alpha\,\mathbb{E}\!\left[ x_1 - x_0 \;\middle|\; \tilde{x}_\tau = x \right].
\]
Since $\tilde{x}_\tau$ is the linear interpolant at fraction $\alpha\tau$, it has the same distribution as $x_t$ with $t = \alpha\tau$. Substituting into~\eqref{eq:vfm} and comparing with~\eqref{eq:vsec_opt}:
\[
    u_\tau^{(\alpha)}(x) = \alpha\, v_{\theta^*}(x, \alpha\tau, 1) = v_{\theta^*}(x, \tau, \alpha).
\]
By uniqueness of solutions to the continuity equation under the Lipschitz assumption, the induced path $q_\tau^{(\alpha)}$ coincides with the marginal of $\tilde{x}_\tau$.
Setting $\tau = 0$ and $\tau = 1$ gives~\eqref{eq:marginal_form}.
\end{proof}
\paragraph{Proof of Theorem~\ref{thm:truncation}}
\begin{proof}
\textbf{Part (i).} The bound~\eqref{eq:euler} follows from the discrete Gr\"onwall inequality applied to the Euler discretization error, as used in the flow matching error analysis of Benton et al.~\cite{benton2024error}. The local truncation error of a single step of size $h$ from $\psi_s$ is $\frac{h^2}{2}\|\mathbf{a}(\psi_s, s)\| + O(h^3)$, and the global error accumulates via the Lipschitz constant $L$.
 
\medskip
\noindent\textbf{Part (ii).} Set $\alpha = 1 - \epsilon$ for small $\epsilon > 0$. Then $\tau = t/(1-\epsilon) = t + t\epsilon + O(\epsilon^2)$. Taylor-expanding the consistency residual:
\[
    v_\theta(x_t, \tau, \alpha) - \alpha\, v_\theta(x_t, t, 1)
    \;=\;
    \epsilon\!\left[\,t\,\partial_t v_\theta(x_t, t, 1) + v_\theta(x_t, t, 1) - \partial_\alpha v_\theta(x_t, t, 1)\,\right]
    + O(\epsilon^2).
\]
Therefore $\mathcal{L}_{\textup{cons}}$ penalizes, at leading order, $\| t\,\partial_t v_\theta + v_\theta - \partial_\alpha v_\theta \|^2$. Setting this to zero yields the cross-scale coherence condition, which constrains $\|\partial_t v_\theta\|$ for bounded $\|\partial_\alpha v_\theta\|$ and $\|v_\theta\|$. Since $\partial_t v_\theta$ is one of the two components of the acceleration, reducing it tightens the global error bound~\eqref{eq:euler}. Vanilla FM trains only with $\mathcal{L}_{\textup{pri}}$, which regresses $v_\theta$ onto $v^*$ pointwise at each $t$ without coupling different times, and therefore imposes no constraint on $\partial_t v_\theta$.
\end{proof}

\begin{corollary}[Consistent Few-Step Sampling]
\label{cor:few_step}
Let $v_{\theta^*}$ be as in Proposition~\ref{prop:convergence}.
 
\begin{enumerate}
    \item[\textbf{(i)}] For any $\alpha \in (0, 1]$, a single Euler step of size $\alpha$ along the \emph{primary} trajectory at $t = 0$ yields a sample whose distribution matches the endpoint of the secondary trajectory:
    \begin{equation}
        x_0 + \alpha\, v_{\theta^*}(x_0, 0, 1)
        \;\stackrel{d}{=}\;
        x_0 + \int_0^1 v_{\theta^*}\!\bigl(x(\tau), \tau, \alpha\bigr)\, d\tau
        \;\sim\; p_\alpha.
    \end{equation}
    Consequently, the family of secondary trajectories indexed by $\alpha$ provides a continuum of self-consistent few-step samplers, each producing a valid interpolant marginal $p_\alpha$ regardless of the discretization granularity.
 
    \item[\textbf{(ii)}] By Theorem~\ref{thm:truncation}, one-step generation is exact if and only if the trajectory acceleration $\mathbf{a} \equiv 0$. Since $\mathcal{L}_{\textup{cons}}$ directly penalizes $\|\partial_t v_\theta\|$, a component of $\|\mathbf{a}\|$, RecFM actively drives the trajectory toward the zero-curvature regime where few-step Euler integration is accurate. Vanilla FM imposes no such constraint, leaving trajectory curvature uncontrolled.
\end{enumerate}
\end{corollary}

\section{Additional Results}\label{appendix:additional_ablation}

\subsection{Influence of Inference Steps}\label{appendix:mse_vs_step}

We further analyze the effect of the number of inference steps in RecFM. Figure~\ref{fig:mse_vs_steps} shows the MSE as a function of the number of inference steps on the Navier-Stokes task. RecFM achieves its best performance with one- or two-step generation. We hypothesize that this behavior is due to the largely deterministic nature of physics-governed systems, where longer sampling trajectories can introduce accumulated errors.

\begin{figure}[H]
    \centering
    \includegraphics[width=0.6\linewidth]{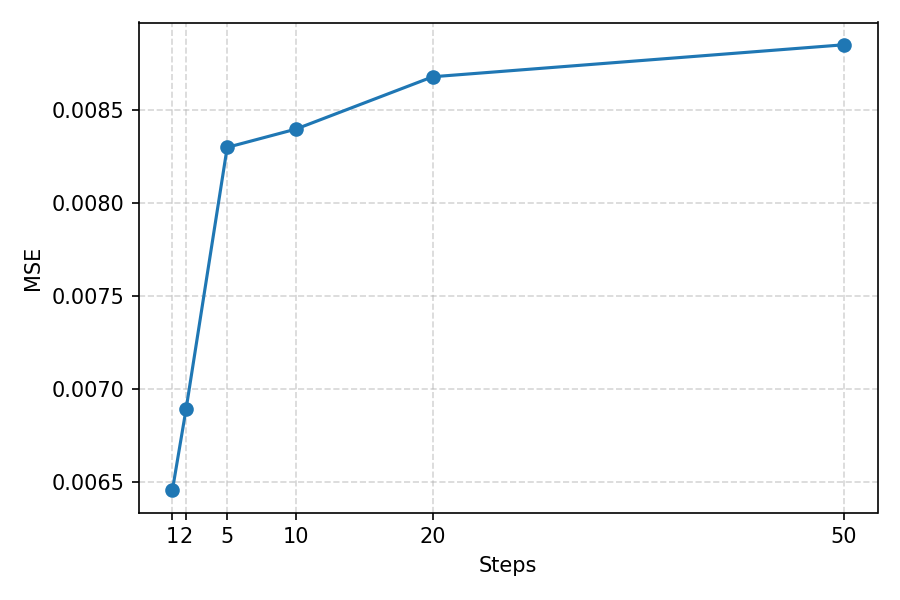}
    \vspace{-0.5em}
    \caption{\textbf{MSE vs. inference steps on the Navier-Stokes benchmark.} RecFM achieves optimal performance with one- and two-step generation, while increasing the number of steps leads to error accumulation.}
    \vspace{-1em}
    \label{fig:mse_vs_steps}
\end{figure}

\subsection{Influence of Recursion Depth $D$}\label{appendix:depth}

We study the effect of recursion depth $D$ on model performance. The depth $D$ controls the number of trajectory scales used during training, thereby governing the strength of multi-scale supervision. While RecFM particularly corresponds to the depth-$2$ case in our main experiments, higher depths introduce additional consistency constraints across more trajectory scales.

Table~\ref{tab:ablation_depth_ns} presents a comparison of different recursion depths on the Navier-Stokes benchmark. The vanilla FM ($D=1$) needs multi-step inference (5 steps) to reach acceptable performance, whereas RecFM ($D=2$) already achieves strong results with a single inference step. 
Raising the recursion depth to $D=3$ leads to slightly inferior performance compared to $D=2$, with small declines in MSE, SSR, and inference speed. In addition, training the depth-3 model requires larger memory than training the depth-2 RecFM due to the extra gradient terms introduced at depth 3. These observations indicate that enforcing pairwise consistency is sufficient to obtain the advantages of multi-scale alignment, and that further increasing the recursion depth yields diminishing returns.

Based on these results, we adopt $D=2$ as the default configuration in the main experiments, as it achieves strong performance while maintaining a simple and efficient formulation.

\begin{table}[!ht]
    \centering
    \caption{Ablation study on recursion depth $D$ for Navier-Stokes flow. Vanilla FM uses 5-step inference, while RecFM variants operate in the 1-step regime.}
    \label{tab:ablation_depth_ns}
    \small
    \begin{tabular}{lcccc}
        \toprule
        \textbf{Depth $D$} & \textbf{CRPS} ($\downarrow$) & \textbf{MSE} ($\downarrow$) & \textbf{SSR} ($\rightarrow 1$) & \textbf{Time [s]} \\
        \midrule
        $D=1$ (Vanilla FM, 5-step) 
        & 0.036 & 0.0076 & 0.911 &  6.914 \\
        
        $D=2$ (RecFM, 1-step) 
        & 0.031 & 0.0064 & 0.959 & 1.588 \\
        
        $D=3$ (extended RecFM, 1-step) 
        &  0.031 & 0.0065 & 1.091 & 1.594 \\
        \bottomrule
    \end{tabular}
\end{table}

\subsection{Additional Training Dynamics and Convergence}

We further compare validation MSE versus NFE during training for flow matching methods. Figure~\ref{fig:mse_vs_nfe_vanilla} shows that RecFM converges faster than Vanilla FM and maintains lower validation error, demonstrating improved efficiency and stability.

\begin{figure}[H]
    \centering
    \includegraphics[width=0.5\linewidth]{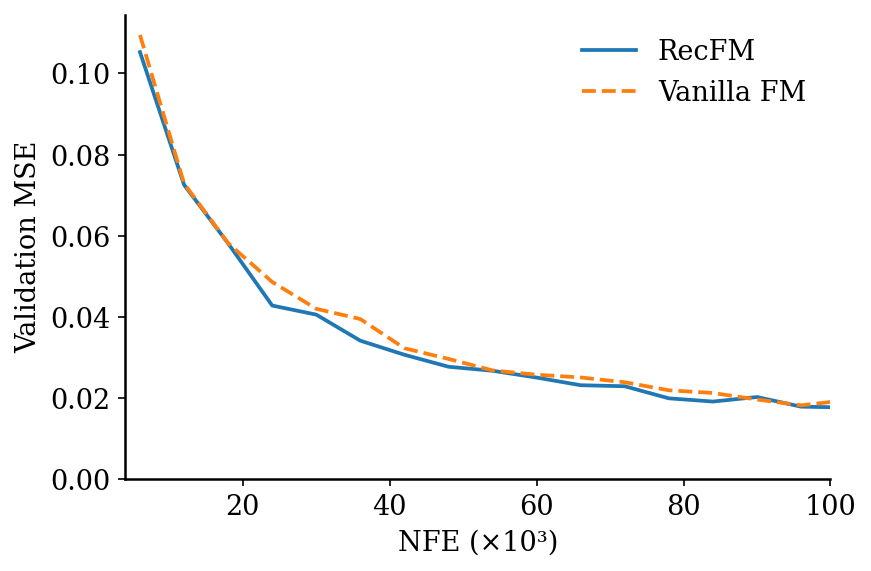}
    \caption{\textbf{Flow matching validation MSE versus NFE during training.} RecFM converges faster than Vanilla FM and maintains consistently lower validation error.}
    \label{fig:mse_vs_nfe_vanilla}
\end{figure}

\section{Architecture and Implementation Details}\label{appendix:architecture}

We adopt the state-of-the-art Hierarchical Video Diffusion Transformer (HV-DiT) backbone from VideoPDE \cite{li2025videopde}, with the sole modification that the input mask channel is removed. Unlike DYffusion \cite{ruhling2023dyffusion}, whose forecasting-oriented formulation is not naturally compatible with transformer-based diffusion backbones such as HV-DiT, RecFM operates directly on the learned velocity field and can be integrated into existing spatiotemporal diffusion architectures with minimal modification. This also makes direct comparison against the original DYffusion architecture a more appropriate and fair evaluation setting. 

During training, our recursive formulation requires $D$ forward/backward gradient evaluations per iteration due to multi-scale trajectory supervision. To maintain comparable overall training cost, we reduce the total number of training iterations by the same factor $D$ relative to other models.

A detailed overview of the model architecture and its hyperparameters (\textit{e.g.}, for the Navier-Stokes equation with $D=2$) is provided in Table~\ref{tab:hyperparameters}.

\begin{table}[ht]
    \centering
    \caption{\textbf{RecFM training and model hyperparameters of Navier-Stokes Flow.}}
    \label{tab:hyperparameters}
    \begin{tabular}{
        l
        c
    }
        \toprule
        Hyperparameter & RecFM (NS Flow) \\
        \midrule
        Parameters & 116.2M \\
        Training steps  & 40k \\
        Batch size & 64 \\
        GPUs & 4 \(\times\) L40S \\
        Mixed Precision & bfloat16 \\
        \midrule
        Patch Size (\(T\times H\times W\)) & [2, 2, 1] \\
        Neighborhood Attention Levels & 1 \\
        Global Attention Level & 1 \\
        Neighborhood Attention Depth & 2 \\
        Global Attention Depth & 11 \\
        Feature Dimensions & [384, 768] \\
        Attention Head Dimension & 64 \\
        Neighborhood Kernel Size (\(T\times H\times W\)) & [2, 7, 7] \\
        Mapping Depth & 1 \\
        Mapping Width & 768 \\
        Dropout & 0 \\
        \midrule
        Optimizer & AdamW \\
        Learning Rate & \(5\times 10^{-4}\) \\
        \([\beta_1,\beta_2]\) & \([0.9, 0.95]\) \\
        $\lambda$ & 1 \\
        Epsilon & \(1\times 10^{-8}\) \\
        Weight Decay & \(1\times10^{-2}\) \\
        \bottomrule
    \end{tabular}
\end{table}

\newpage
\section{Additional Visualizations}\label{appendix:vis}
\vspace{-0.2em}

Visualization of more timesteps of the Helmholtz Staircase equation is shown in Figure \ref{fig:helmholtz_channels}. Additionally, we visualize a representative Navier-Stokes rollout in Figure \ref{fig:ns_rollout}.

\vspace{-0.5em}
\begin{figure}[H]
    \centering
    \includegraphics[width=0.999\linewidth]{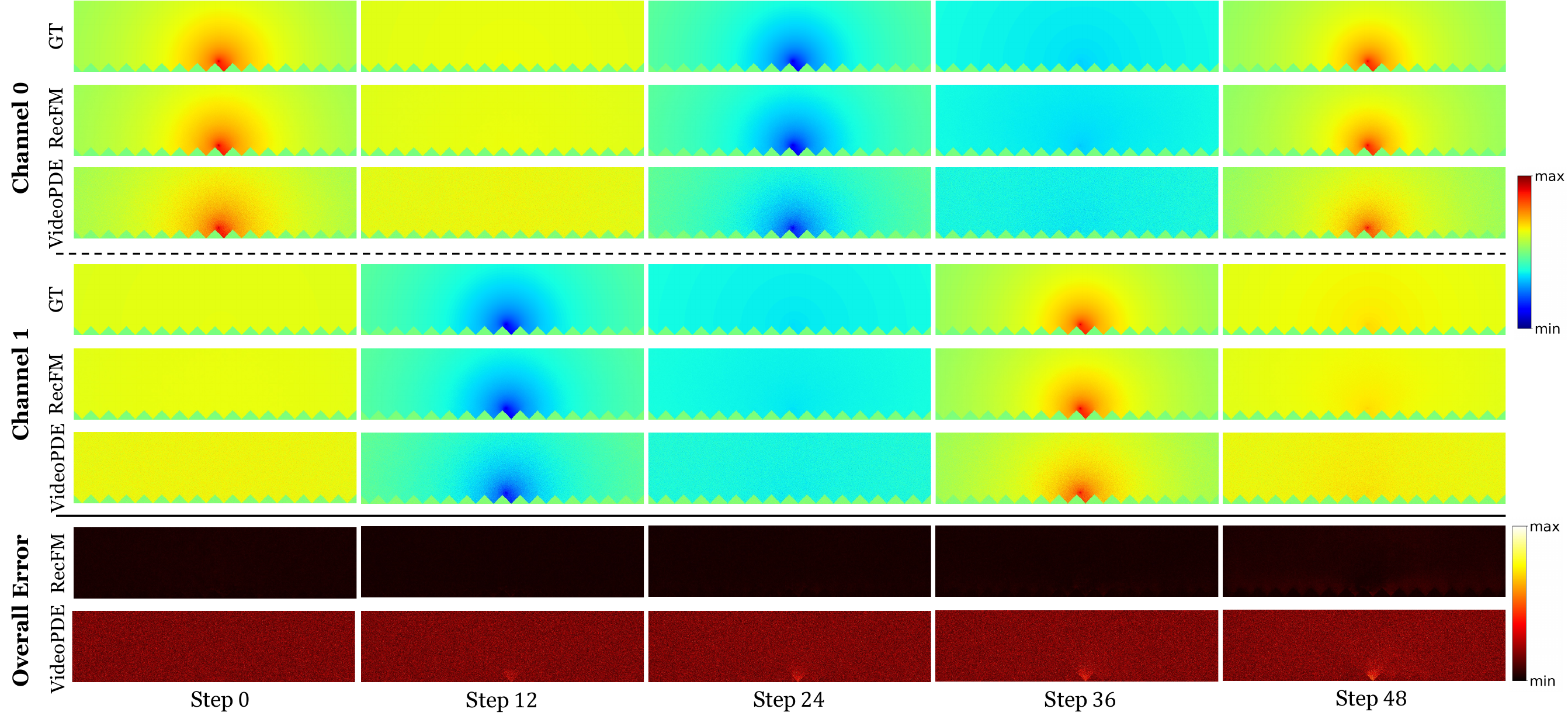}
    \caption{\textbf{More roll-out results of the Helmholtz Staircase equation.} Visual comparison of Ground Truth against RecFM and VideoPDE (best-performed baseline) for two channels, with the bottom rows indicating absolute errors. Columns correspond to dataset timesteps.}
    \vspace{-0.5em}
    \label{fig:helmholtz_channels}
\end{figure}

\begin{figure}[H]
    \centering
    \includegraphics[height=0.41\linewidth, angle=-90]{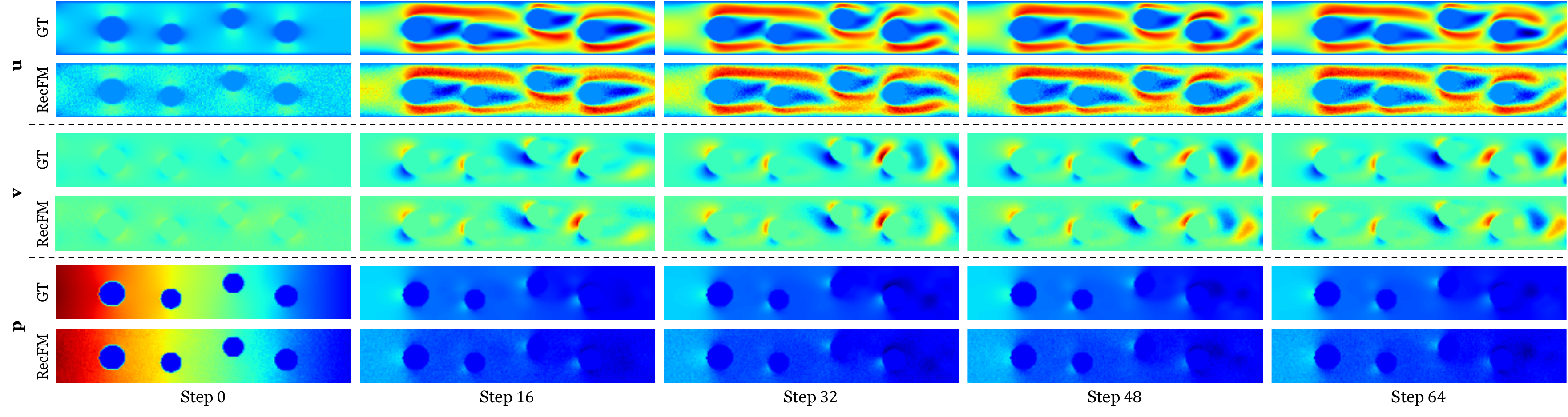}
    \caption{\textbf{Navier-Stokes rollout sample.}}
    \label{fig:ns_rollout}
\end{figure}

\section{Physics-Informed Evaluation} \label{appendix:pbfm}

We further compare RecFM with PBFM \cite{pbfm2026} on PDE-governed datasets in Table~\ref{tab:pbfm}, following the original setting of 20 inference steps. While PBFM consistently improves over vanilla Flow Matching, its iterative refinement procedure limits both efficiency and accuracy relative to RecFM. In contrast, RecFM achieves higher accuracy while requiring only 1-2 inference steps, resulting in substantially faster generation.

To assess physical consistency, we additionally report physics-informed evaluation metrics. For the Navier-Stokes dataset, the trajectories are transient and do not reach a statistically stationary regime, making standard long-time turbulence diagnostics inapplicable. Instead, we evaluate the average kinetic energy $\langle E(t) \rangle$, normalized by the initial ground-truth energy $\langle E^{t=0}_{\text{real}} \rangle$. In Table \ref{tab:pbfm}, we report KE Accuracy, which measures the relative agreement between predicted and ground-truth energy (values closer to 1 indicate better physical fidelity). The rollout of kinetic energy is shown in Figure \ref{fig:ke_rollout}. While PBFM enforces strong physical constraints at each autoregressive iteration, errors accumulate over time. In contrast, RecFM maintains stable dynamics and achieves lower overall error without noticeable accumulation.

For the Helmholtz Staircase equation, we further report the PDE residual of the wave equation $\partial^2 U / \partial t^2 + \omega^2 U = 0$ in the table, where values closer to zero indicate better adherence to the governing dynamics.

These results demonstrate that RecFM not only produces visually accurate predictions, but also more faithfully preserves the underlying physical dynamics due to its few-step nature.

\vspace{-0.5em}

\begin{table}[H]
    \centering
    \setlength{\tabcolsep}{3.3pt}
    \caption{Physics-informed quantitative forecasting results for Navier-Stokes Flow, and Helmholtz Staircase Equation. Lower values are better for MSE, CRPS, and PDE Residual, while SSR and KE Accuracy are optimal when closer to 1. Best results in \textbf{bold}.}
    \label{tab:pbfm}
    \small
    \begin{tabular}{l cccc c cccc}
        \toprule
        \multirow{2}{*}{\textbf{Method}}  & \multicolumn{5}{c}{\textbf{Navier-Stokes}} & \multicolumn{3}{c}{\textbf{Helmholtz Staircase}}\\
        \cmidrule(lr){2-6} \cmidrule(lr){7-10} 
        & \textbf{CRPS} & \textbf{MSE} & \textbf{SSR} & \textbf{KE Accuracy} & \textbf{Time [s]} & \textbf{CRPS} & \textbf{MSE} & \textbf{SSR} & \textbf{PDE Residual}  \\
        \midrule
        VideoPDE \cite{li2025videopde} & 0.033 & 0.0068 & 0.205 & 0.9670 & 72.64 & 0.026 & 5.6e-4 & 4.334 & 0.00841 \\
        Vanilla FM  & 0.036 & 0.0076 & 0.911 & 0.9522 & 6.914 & 0.030 & 6.5e-4 & 1.485 & 0.01102 \\
        PBFM \cite{pbfm2026} & 0.034 & 0.0071 & 0.810 & 0.9592 & 14.75 & 0.0094 & 1.2e-4 & 0.737 & 0.00519 \\
        \midrule
        RecFM (1-step) & \textbf{0.031} & \textbf{0.0064} & \textbf{0.959} & \textbf{0.9791} & 1.588 & 0.0034 & 4.2e-5 & \textbf{1.090} & 0.00476 \\
        RecFM (2-step) &  0.032 &  0.0068 &  0.932 & 0.9672 & 3.128 & \textbf{0.0027} & \textbf{2.7e-5} & 1.440 & \textbf{0.00457} \\
        \bottomrule
    \end{tabular}
\end{table}

\vspace{-1em}

\begin{figure}[H]
    \centering
    \includegraphics[width=0.9\linewidth]{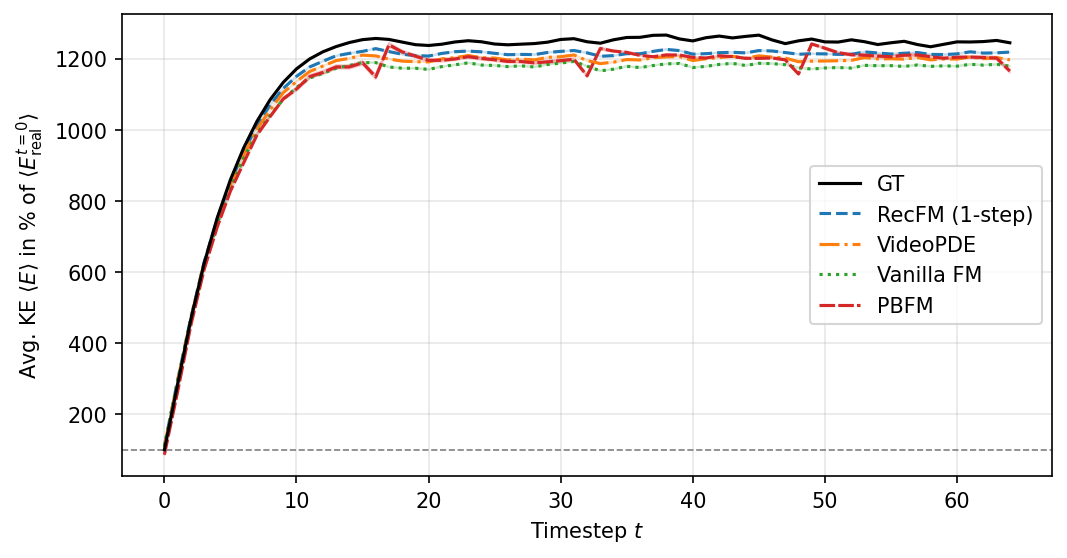}
    \caption{\textbf{Average kinetic energy over time.} $\langle E (t)\rangle$ normalized by $\langle E^{t=0}_{\text{real}} \rangle$. $100\%$ corresponds to the initial ground-truth energy.}
    \label{fig:ke_rollout}
\end{figure}

\vspace{-1.5em}

\newpage
\section{Statistical Significance}\label{appendix:stdev}
\vspace{-0.4em}
We additionally report the standard deviation of all metrics (CRPS, MSE, and SSR) for RecFM, VideoPDE, and Vanilla FM (\textit{i.e.}, all methods using the HV-DiT backbone, as shown in Tables \ref{tab:stdev_sst_ns} and \ref{tab:stdev_helmholtz}.

\begin{table}[H]
\centering
\setlength{\tabcolsep}{4pt}
\caption{Quantitative forecasting results (mean $\pm$ std) on SST and Navier--Stokes datasets. Lower CRPS and MSE are better, while SSR closer to 1 indicates better calibration.}
\label{tab:stdev_sst_ns}
\small
\begin{tabular}{l ccc ccc}
\toprule
\multirow{2}{*}{\textbf{Method}} & \multicolumn{3}{c}{\textbf{SST}} & \multicolumn{3}{c}{\textbf{Navier--Stokes}} \\
\cmidrule(lr){2-4} \cmidrule(lr){5-7}
 & CRPS $\downarrow$ & MSE $\downarrow$ & SSR $\to 1$ 
 & CRPS $\downarrow$ & MSE $\downarrow$ & SSR $\to 1$ \\
\midrule
VideoPDE \cite{li2025videopde} 
& $0.216{\scriptstyle \pm 0.005}$ & $0.162{\scriptstyle \pm 0.002}$ & $0.746{\scriptstyle \pm 0.007}$
& $0.033{\scriptstyle \pm 0.002}$ & $0.0068{\scriptstyle \pm 0.0003}$ & $0.205{\scriptstyle \pm 0.013}$ \\

Vanilla FM 
& $0.260{\scriptstyle \pm 0.007}$ & $0.232{\scriptstyle \pm 0.002}$ & $0.914{\scriptstyle \pm 0.006}$
& $0.036{\scriptstyle \pm 0.001}$ & $0.0076{\scriptstyle \pm 0.0001}$ & $0.911{\scriptstyle \pm 0.008}$ \\

\midrule
RecFM (1-step) 
& $0.217{\scriptstyle \pm 0.003}$ & $0.162{\scriptstyle \pm 0.001}$ & $0.984{\scriptstyle \pm 0.006}$
& $0.031{\scriptstyle \pm 0.001}$ & $0.0064{\scriptstyle \pm 0.0001}$ & $0.959{\scriptstyle \pm 0.003}$ \\

RecFM (2-step) 
& $0.216{\scriptstyle \pm 0.003}$ & $0.161{\scriptstyle \pm 0.001}$ & $1.004{\scriptstyle \pm 0.007}$
& $0.032{\scriptstyle \pm 0.001}$ & $0.0068{\scriptstyle \pm 0.0001}$ & $0.932{\scriptstyle \pm 0.004}$ \\
\bottomrule
\end{tabular}
\end{table}

\begin{table}[H]
\centering
\setlength{\tabcolsep}{5pt}
\caption{Quantitative forecasting results (mean $\pm$ std) on the Helmholtz Staircase dataset. Lower CRPS and MSE are better, while SSR closer to 1 indicates better calibration.}
\label{tab:stdev_helmholtz}
\small
\begin{tabular}{l ccc}
\toprule
\textbf{Method} & CRPS $\downarrow$ & MSE $\downarrow$ & SSR $\to 1$ \\
\midrule
VideoPDE \cite{li2025videopde} 
& $0.026{\scriptstyle \pm 0.001}$ & $5.6\text{e-}4{\scriptstyle \pm 1\text{e-}5}$ & $4.334{\scriptstyle \pm 0.071}$ \\

Vanilla FM 
& $0.030{\scriptstyle \pm 0.001}$ & $6.5\text{e-}4{\scriptstyle \pm 1\text{e-}5}$ & $1.485{\scriptstyle \pm 0.012}$ \\

\midrule
RecFM (1-step) 
& $0.0034{\scriptstyle \pm 0.0001}$ & $4.2\text{e-}5{\scriptstyle \pm 1\text{e-}6}$ & $1.090{\scriptstyle \pm 0.010}$ \\

RecFM (2-step) 
& $0.0027{\scriptstyle \pm 0.0001}$ & $2.7\text{e-}5{\scriptstyle \pm 1\text{e-}6}$ & $1.440{\scriptstyle \pm 0.012}$ \\
\bottomrule
\end{tabular}
\end{table}

\section{Shortcut Models vs. RecFM}
\label{app:shortcut-recfm}

Although both Shortcut Models~\citep{frans2024one} and RecFM enforce self-consistency to enable few-step generation, they impose structurally distinct constraints on the learned velocity field. Shortcut Models parameterize the network by a step-size $d$ and enforce a \emph{compositional} consistency condition: a single step of size $2d$ must produce the same result as two consecutive steps of size $d$,
\begin{equation}
    x_t + 2d \cdot v_\theta(x_t, t, 2d) = \bigl(x_t + d \cdot v_\theta(x_t, t, d)\bigr) + d \cdot v_\theta\!\bigl(x_t + d \cdot v_\theta(x_t, t, d),\; t+d,\; d\bigr).
    \label{eq:shortcut-consistency}
\end{equation}
Crucially, the right-hand side evaluates $v_\theta$ at the \emph{step-forward} state $x_t + d \cdot v_\theta(x_t, t, d)$, coupling the constraint to the spatial geometry of the trajectory. In contrast, RecFM conditions the network on a continuous scale parameter $\alpha$ and enforces a \emph{pointwise scaling} relation at the same spatial location $x_t$:
\begin{equation}
    \mathcal{L}_{\mathrm{cons}} = \bigl\| v_\theta(x_t,\, \tau,\, \alpha) - \alpha \, v_\theta(x_t,\, t,\, 1) \bigr\|_2^2, \qquad \tau = t / \alpha.
    \label{eq:recfm-consistency}
\end{equation}
No choice of the hyperparameters $(\alpha, \lambda)$ can reduce Equation~\ref{eq:recfm-consistency} to Equation~\ref{eq:shortcut-consistency}, because the former is evaluated at a single point while the latter intrinsically depends on a forward Euler update.

Despite this structural gap, the two constraints become locally equivalent in an infinitesimal limit. Setting $\alpha = 1 - \epsilon$ in the RecFM constraint and noting that $\tau = t/(1-\epsilon) \approx t + t\epsilon$, a first-order Taylor expansion of the left-hand side of Equation~\ref{eq:recfm-consistency} around $(t, 1)$ gives $v_\theta(x_t, t, 1) + t\epsilon\,\partial_t v_\theta - \epsilon\,\partial_\alpha v_\theta$, while the right-hand side becomes $(1-\epsilon)\,v_\theta(x_t, t, 1)$. Equating and dividing by $\epsilon$ yields the leading-order constraint
\begin{equation}
    t\,\partial_t v_\theta(x_t, t, 1) - \partial_\alpha v_\theta(x_t, t, 1) + v_\theta(x_t, t, 1) = 0,
    \label{eq:recfm-infinitesimal}
\end{equation}
which constrains how the velocity field varies with time and scale. The analogous expansion of the Shortcut condition~\ref{eq:shortcut-consistency} as $d \to 0^+$ instead produces a constraint involving the spatial Jacobian $\nabla_x v_\theta \cdot v_\theta$, because the forward evaluation point requires a spatial Taylor expansion. Both conditions penalize trajectory curvature at first order, but through complementary mechanisms: Shortcut Models regularize via spatial composition, while RecFM regularizes via cross-scale coherence. At the global optimum of either objective the velocity field recovers the constant OT velocity $v_\theta(x_t, t) = x_1 - x_0$, for which straight-line trajectories trivially satisfy both Equation~\ref{eq:shortcut-consistency} and Equation~\ref{eq:recfm-consistency}.

\paragraph{Intuition for the scaled velocity.} The scaled velocity is never used at inference: generation always proceeds with $\alpha = 1$. Its role is purely as a \emph{training scaffold}. A given interpolated state $x_t = (1-t)x_0 + tx_1$ lies simultaneously on a family of trajectories: the primary trajectory (from $x_0$ to $x_1$, at time $t$) and a secondary trajectory (from $x_0$ to the partial target $x_\alpha = (1-\alpha)x_0 + \alpha x_1$, at rescaled time $\tau = t/\alpha$). While the velocity scales trivially by $\alpha$, the underlying directional estimation problem, recovering $x_1 - x_0$ from the noisy state $x_t$ is shared across all scales. Each $(\tau, \alpha)$ pair therefore provides an independent supervisory signal for the same directional quantity at the same spatial point, functioning as data augmentation in the conditioning space of the network. This is particularly beneficial in the one-step regime, where the entire generation quality depends on a single evaluation of $v_\theta(x_0, 0, 1)$: RecFM enriches the gradient information at every training point through the secondary and consistency losses, whereas vanilla flow matching provides only a single regression target per sample.
Moreover, the flexibility of RecFM in selecting $\alpha$ removes the need for the warm-up phase often required in flow matching and diffusion-based shortcut models, leading to more stable and efficient training.

\paragraph{Performance Comparison on Physics Dynamics.}
In Table~\ref{tab:shortcut_helmholtz}, we compare the Shortcut Model with RecFM on the Helmholtz Staircase dataset for 1-step generation. RecFM achieves better performance, while the Shortcut Model underperforms despite hyperparameters being carefully tuned to the best of our knowledge. We note that the Shortcut Model is primarily designed for static image generation, and extending it to dynamic settings presents additional challenges.

\begin{table}[H]
    \centering
    \caption{1-Step quantitative forecasting results for Helmholtz Staircase Equation.}
    \label{tab:shortcut_helmholtz}
    \small
    \begin{tabular}{l ccc}
        \toprule
        \multirow{2}{*}{\textbf{Method}}  & \multicolumn{3}{c}{\textbf{Helmholtz Staircase}}\\
        \cmidrule(lr){2-4}
        & \textbf{CRPS} & \textbf{MSE} & \textbf{SSR}  \\
        \midrule
        Shortcut Model \cite{frans2024one} & 0.0144 & 1.6e-4 & 0.467 \\
        RecFM & 0.0034 & 4.2e-5 & 1.090 \\
        \bottomrule
    \end{tabular}
\end{table}

\newpage
\section{Recursive Flow Matching for Image Generation}\label{appendix:img}

Table \ref{tab:fid_comparison} presents the image generation performance of our proposed RecFM-XL model compared to other generative baselines. We report the Fréchet Inception Distance (FID) and evaluate all models with classifier-free guidance (CFG = 1.5) on the ImageNet-1k dataset \cite{deng2009imagenet}. RecFM achieves an FID $<3$ within 16 sampling steps. Our results show that RecFM is competitive in image generation as a multi-step flow matching method, while requiring fewer training epochs and inference steps compared with DiT \cite{peebles2023scalable} and SiT \cite{ma2024sit}. Notably, RecFM performs better with 16 inference steps than with 128, which we attribute to its training objective that emphasizes few-step generation, limiting improvements from additional steps.

\begin{table}[!ht]
    \centering
    \caption{Comparison of generative models under different sampling regimes.}
    \label{tab:fid_comparison}
    \small
    \begin{tabular}{lcccc}
        \toprule
        \textbf{Model} & \textbf{FID} $\downarrow$ & \textbf{Sampling Steps} & \textbf{Param Count} & \textbf{Epochs Trained} \\
        \midrule
        DiT-XL \cite{peebles2023scalable} & 2.27 & 500 & 675M & 640 \\
        SiT-XL \cite{ma2024sit} & 2.06 & 250 & 675M & 640 \\
        ADM-G \cite{dhariwal2021diffusion} & 4.59 & 250 & -- & 426 \\
        LDM-4-G \cite{rombach2022high} & 3.6 & 500 & 400M & 106 \\
        Shortcut Model (XL) \cite{frans2024one} & 3.8 & 128 & 676M & 250 \\
        RecFM-XL & 2.53 & 128 & 675M & 160 \\
        RecFM-XL & 2.49 & 16 & 675M & 160 \\
        RecFM-XL & 3.22 & 8 & 675M & 160 \\
        \bottomrule
    \end{tabular}
\end{table}

We further include some visualizations of RecFM-XL model with CFG, as shown in Figures \ref{fig:rec_fm_selected}, \ref{fig:recfm_samples}, \ref{fig:recfm_samples_continued}.

\begin{figure}[!t]
    \centering
    \includegraphics[width=0.99\linewidth]{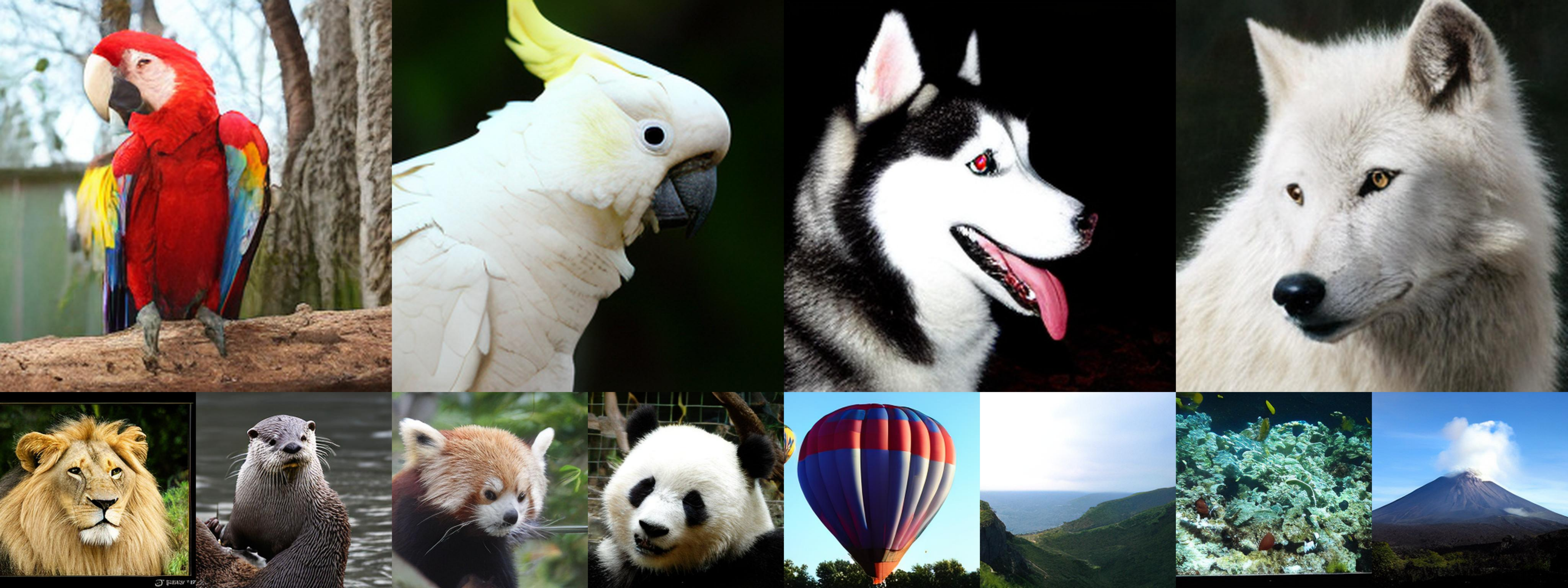}
    \caption{\textbf{Selected samples from our $256\times 256$ resolution RecFM-XL model.}}
    \label{fig:rec_fm_selected}
\end{figure}

\begin{figure}[!h]
    \centering
    \begin{subfigure}[t]{0.48\linewidth}
        \centering
        \includegraphics[width=\linewidth]{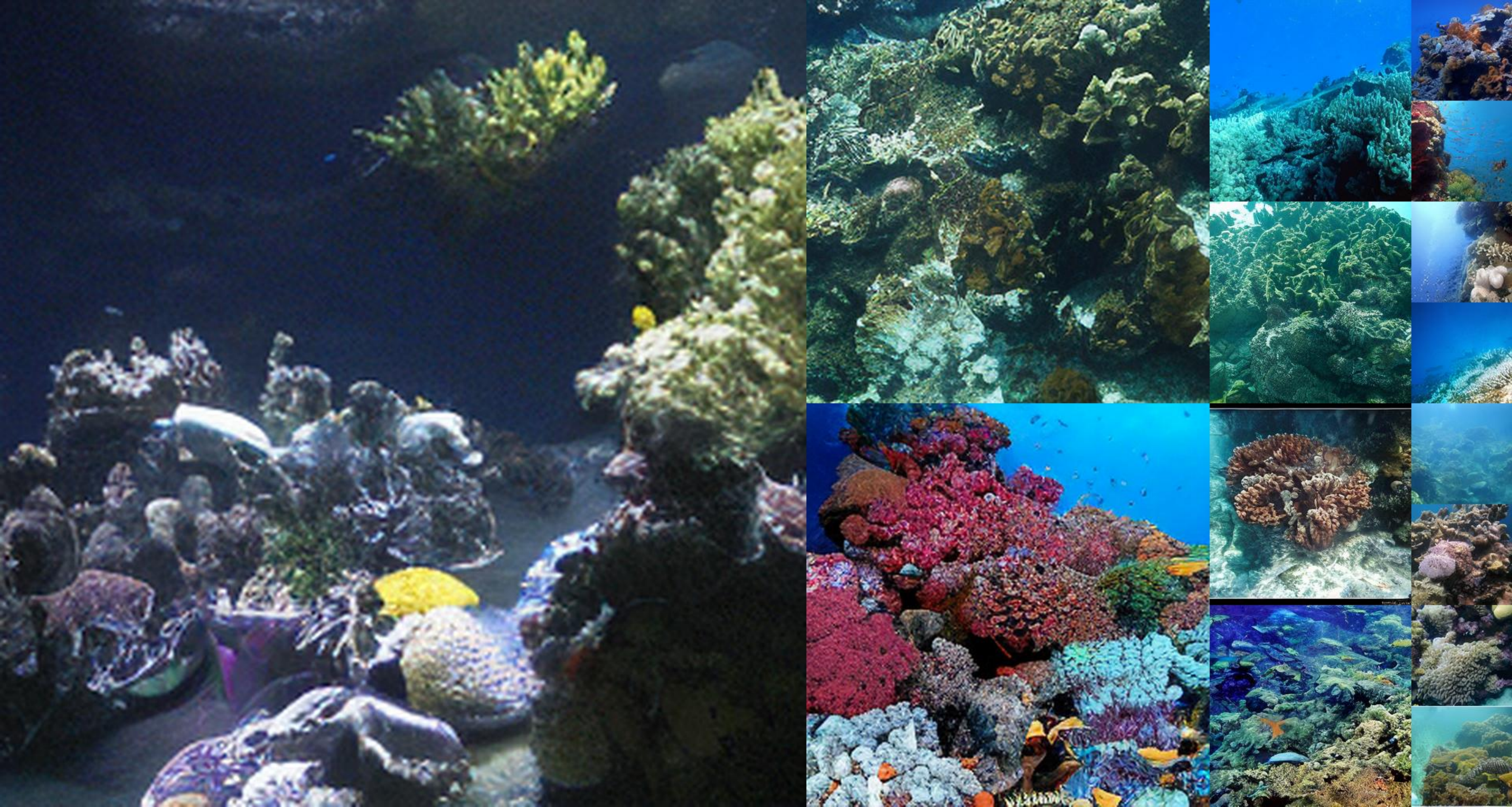}
        \caption{Coral reef (973)}
        \label{fig:rec_fm_973}
    \end{subfigure}
    \hfill
    \begin{subfigure}[t]{0.48\linewidth}
        \centering
        \includegraphics[width=\linewidth]{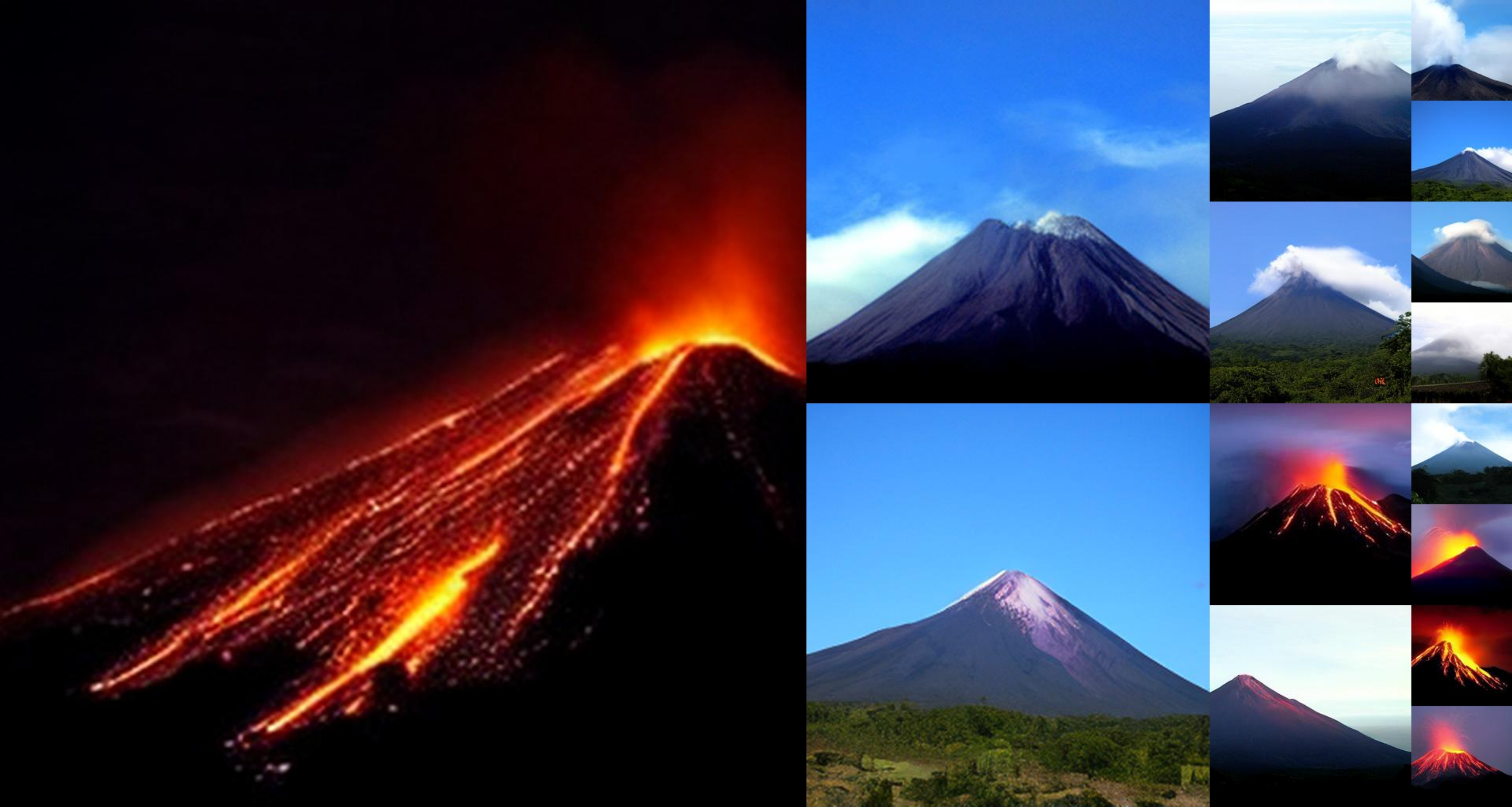}
        \caption{Volcano (980)}
        \label{fig:rec_fm_980}
    \end{subfigure}

        \caption{\textbf{Uncurated $256\times256$ RecFM-XL samples.} Each panel shows samples from a different ImageNet class.}
    \label{fig:recfm_samples}
\end{figure}

\begin{figure}[!h]
    \centering
    \begin{subfigure}[t]{0.48\linewidth}
        \centering
        \includegraphics[width=\linewidth]{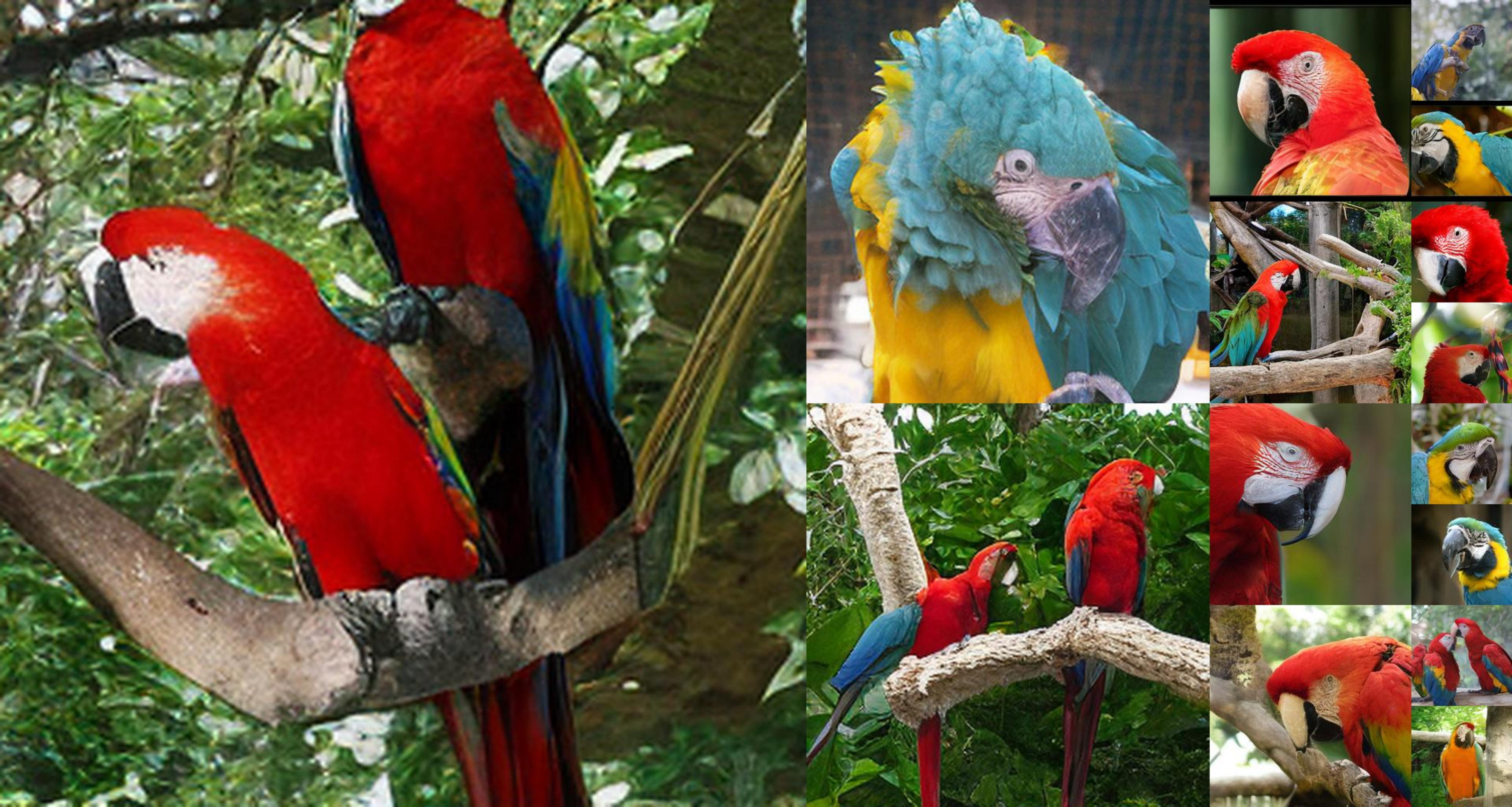}
        \caption{Macaw (88)}
        \label{fig:rec_fm_88}
    \end{subfigure}
    \hfill
    \begin{subfigure}[t]{0.48\linewidth}
        \centering
        \includegraphics[width=\linewidth]{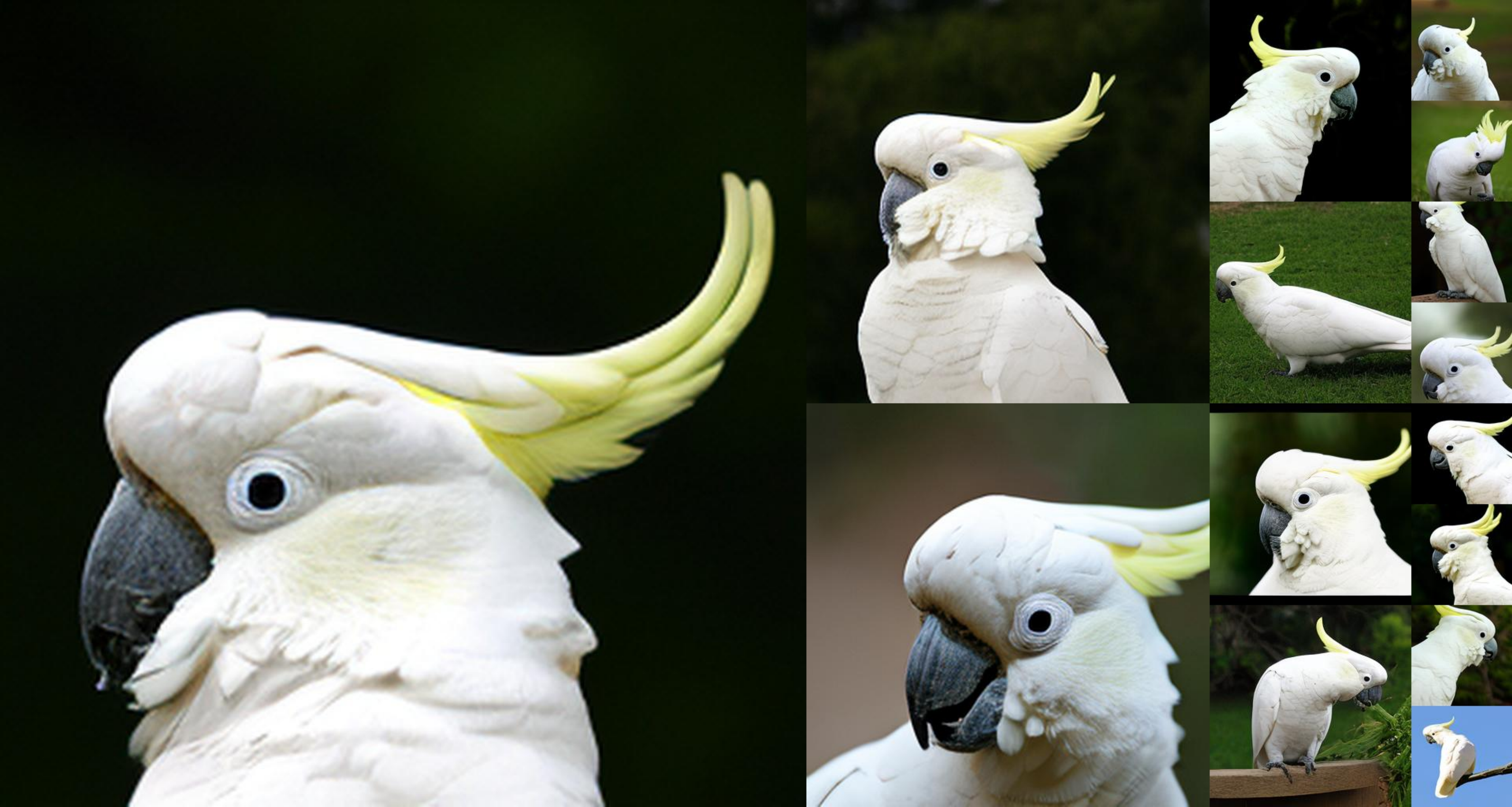}
        \caption{Sulphur-crested cockatoo (89)}
        \label{fig:rec_fm_89}
    \end{subfigure}

    \begin{subfigure}[t]{0.48\linewidth}
        \centering
        \includegraphics[width=\linewidth]{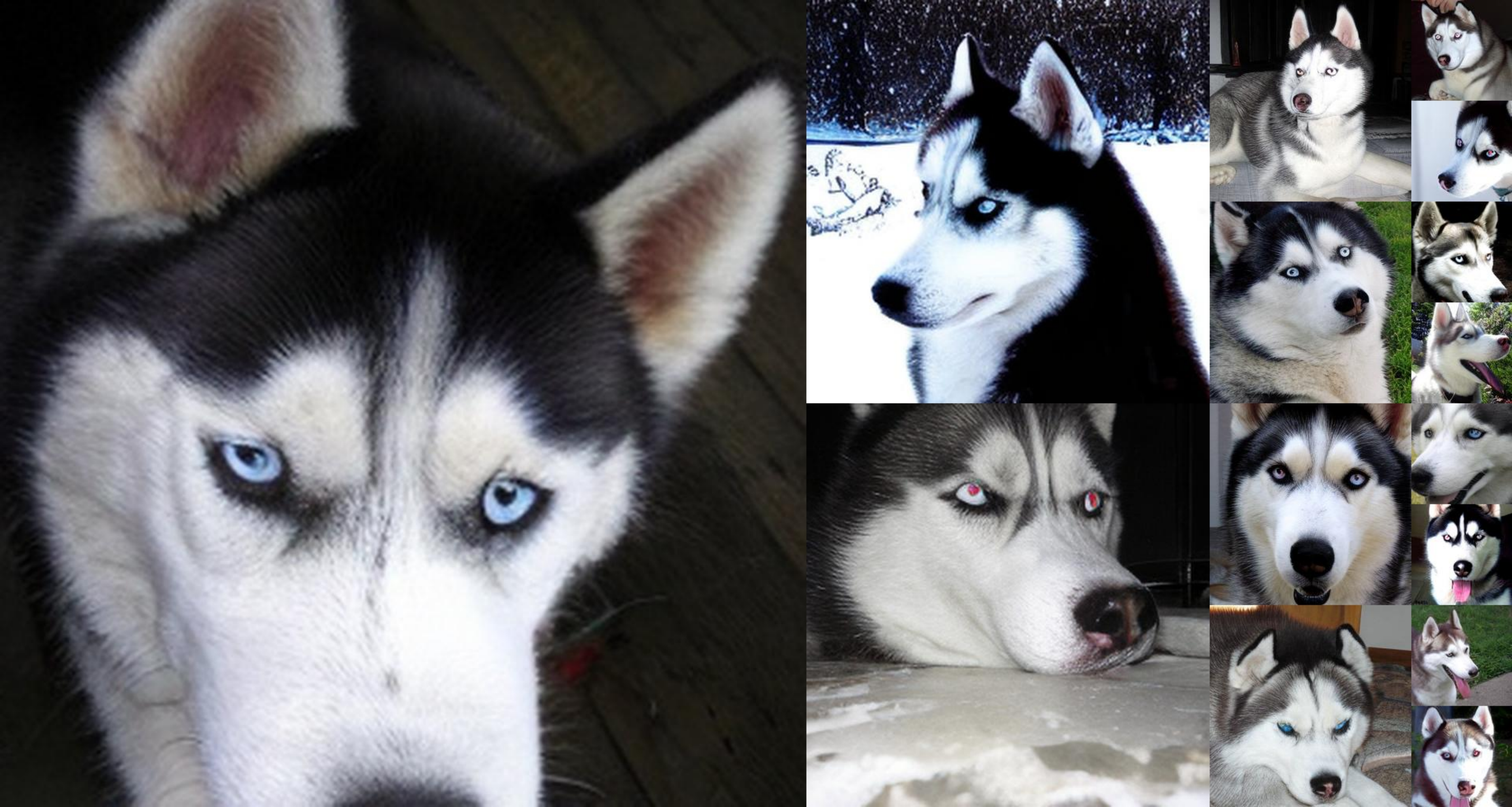}
        \caption{Husky (250)}
        \label{fig:rec_fm_250}
    \end{subfigure}
    \hfill
    \begin{subfigure}[t]{0.48\linewidth}
        \centering
        \includegraphics[width=\linewidth]{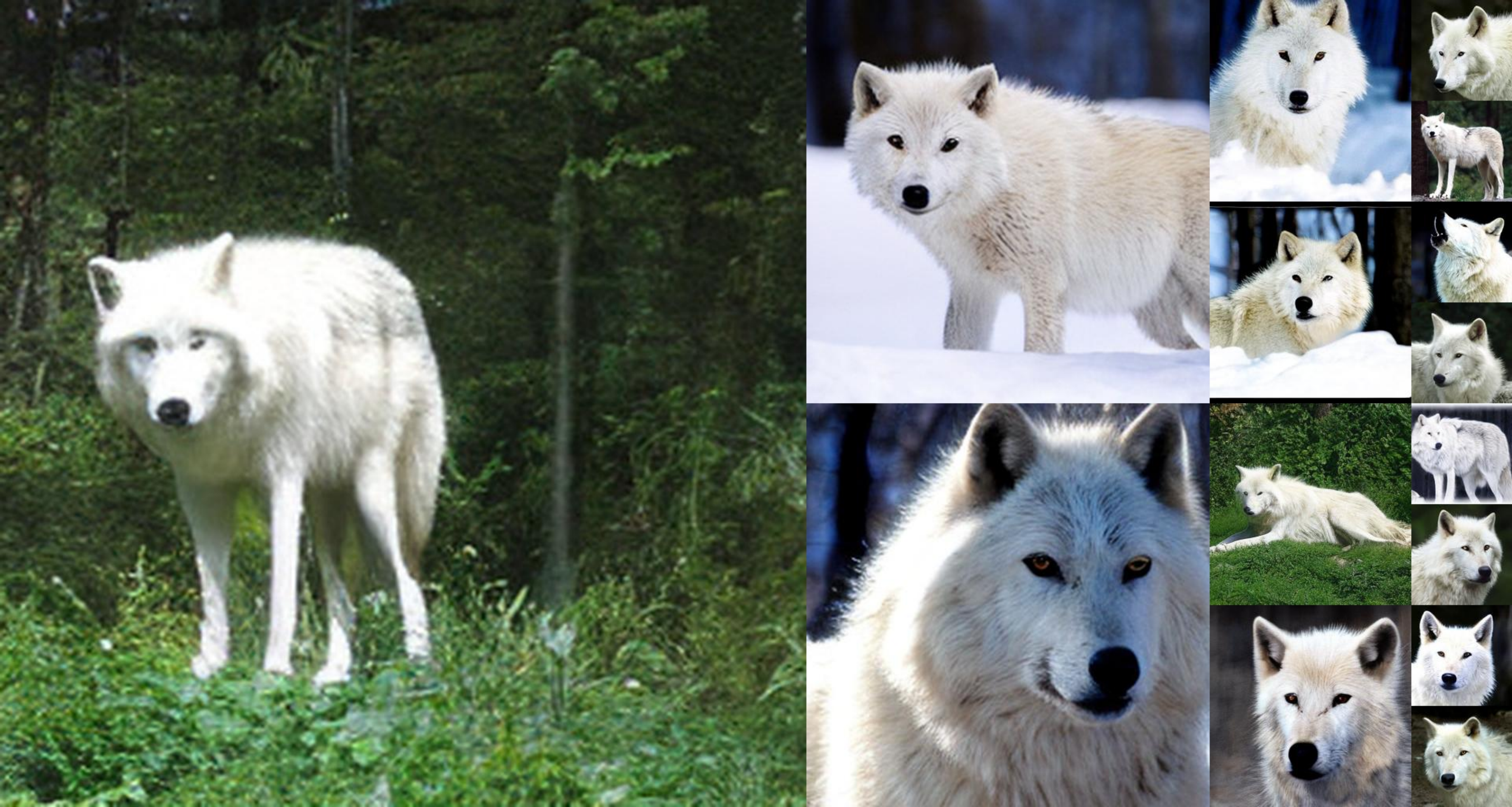}
        \caption{Arctic wolf (270)}
        \label{fig:rec_fm_270}
    \end{subfigure}
    \begin{subfigure}[t]{0.48\linewidth}
        \centering
        \includegraphics[width=\linewidth]{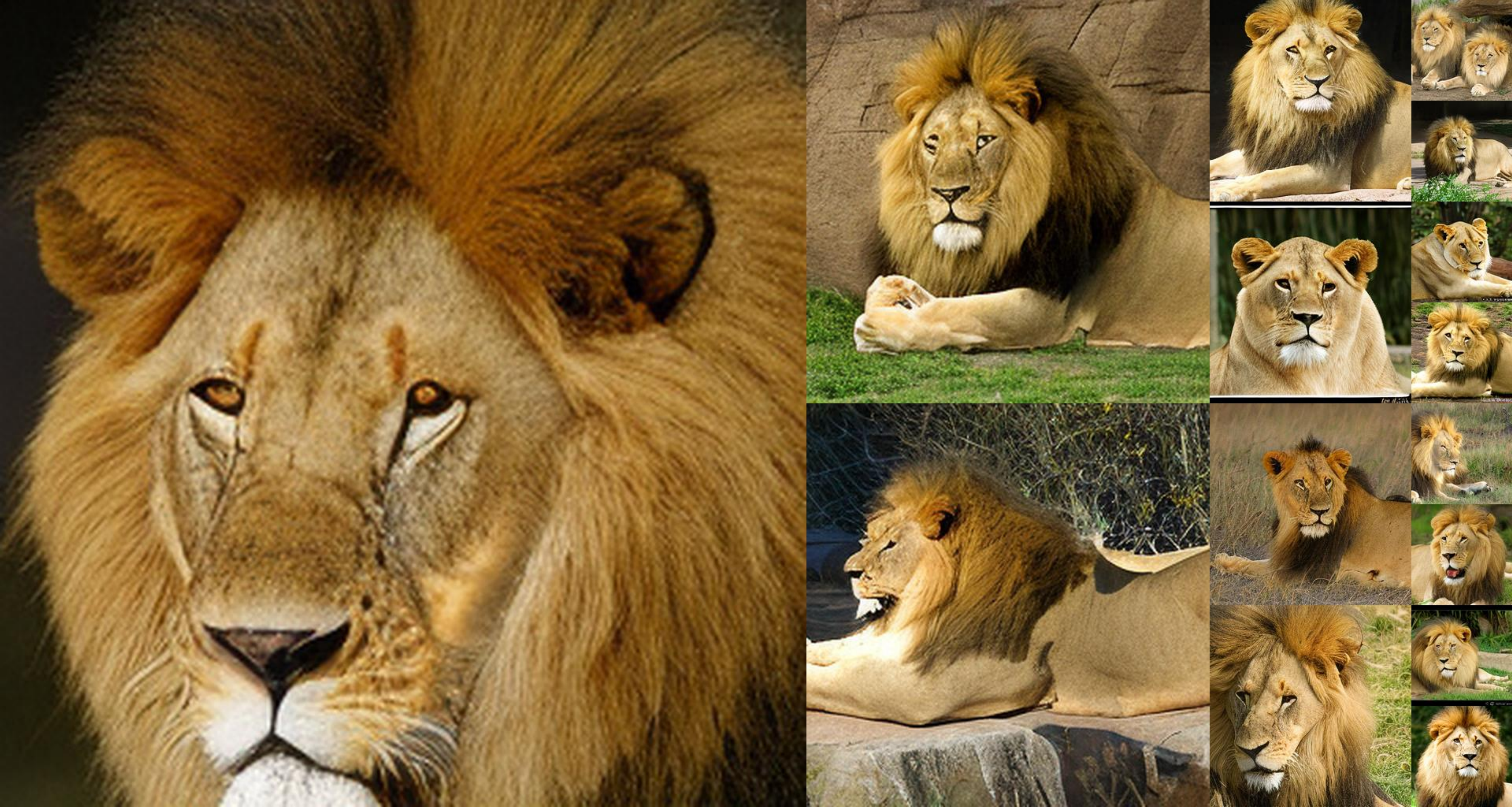}
        \caption{Lion (291)}
        \label{fig:rec_fm_291}
    \end{subfigure}
    \hfill
    \begin{subfigure}[t]{0.48\linewidth}
        \centering
        \includegraphics[width=\linewidth]{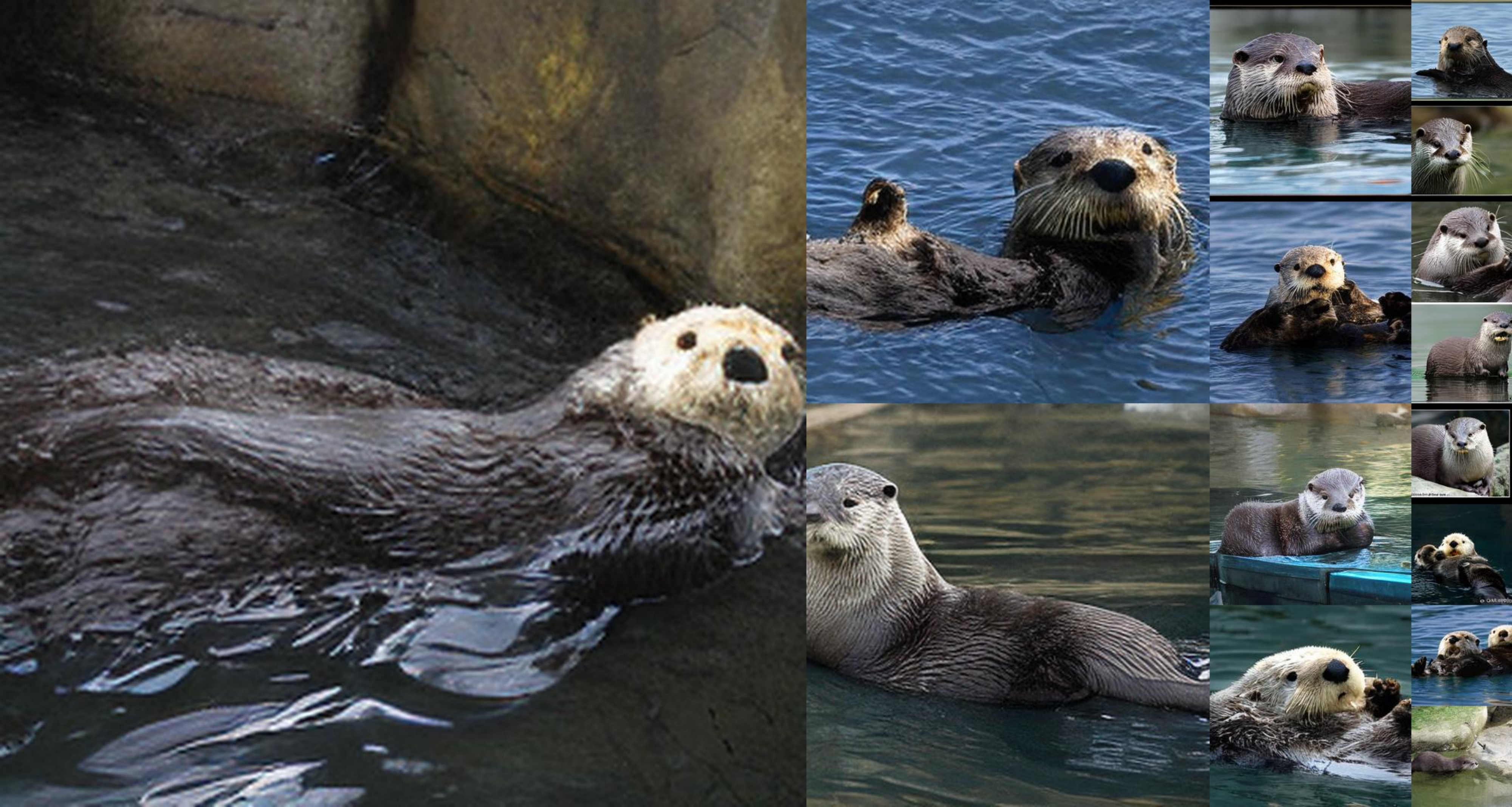}
        \caption{Otter (360)}
        \label{fig:rec_fm_360}
    \end{subfigure}
    \begin{subfigure}[t]{0.48\linewidth}
        \centering
        \includegraphics[width=\linewidth]{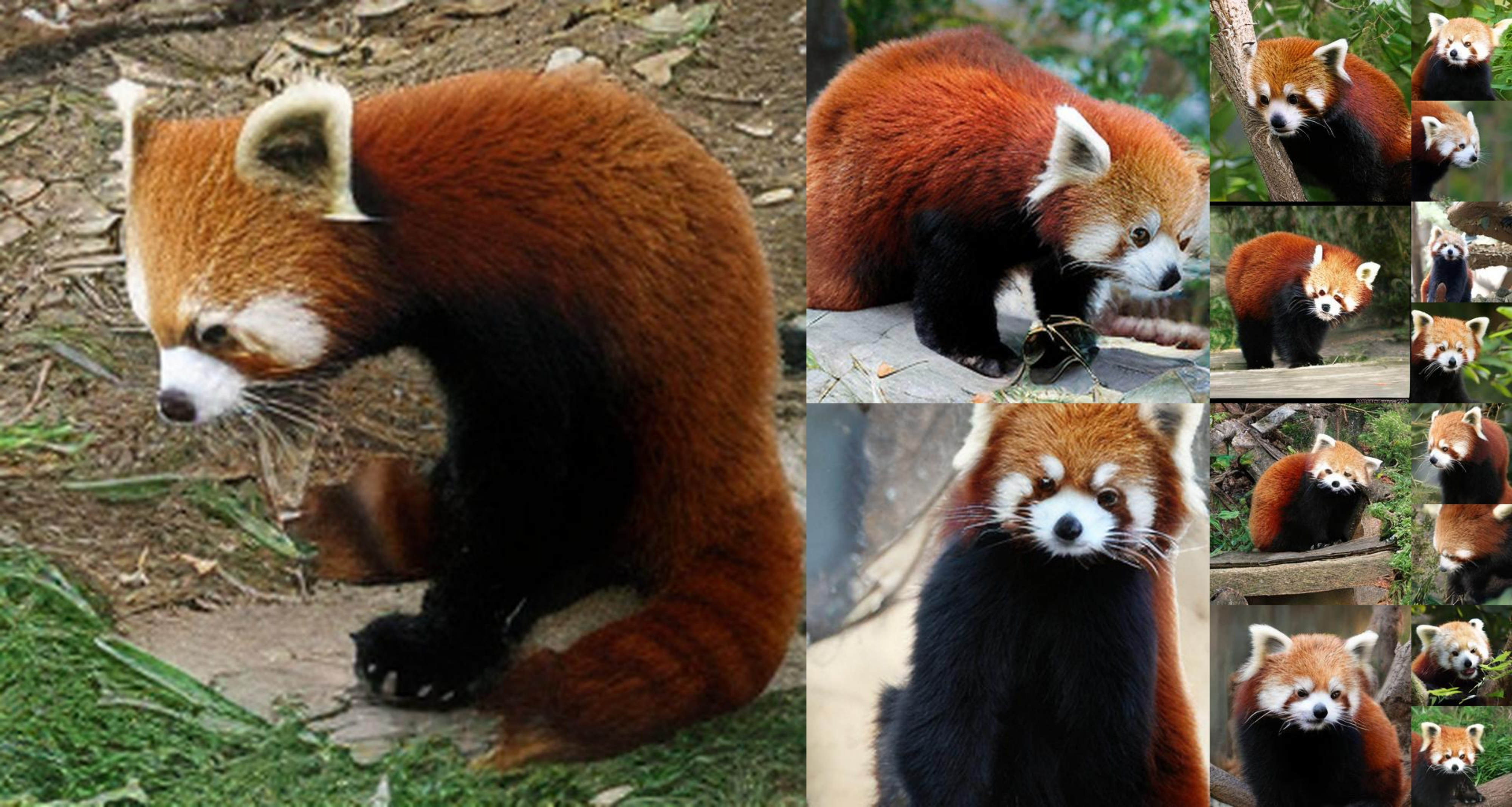}
        \caption{Red panda (387)}
        \label{fig:rec_fm_387}
    \end{subfigure}
    \hfill
    \begin{subfigure}[t]{0.48\linewidth}
        \centering
        \includegraphics[width=\linewidth]{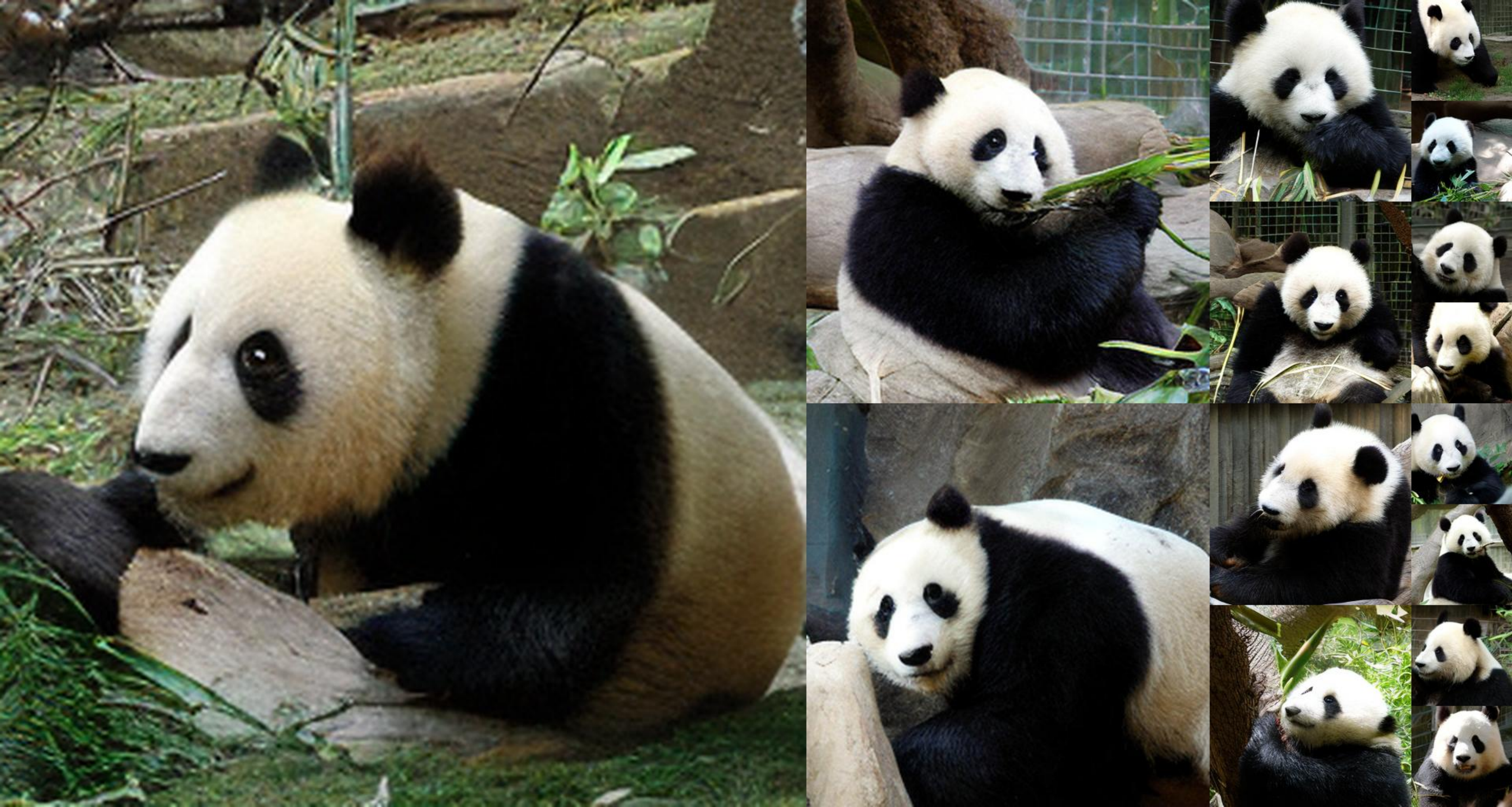}
        \caption{Panda (388)}
        \label{fig:rec_fm_388}
    \end{subfigure}
    \begin{subfigure}[t]{0.48\linewidth}
        \centering
        \includegraphics[width=\linewidth]{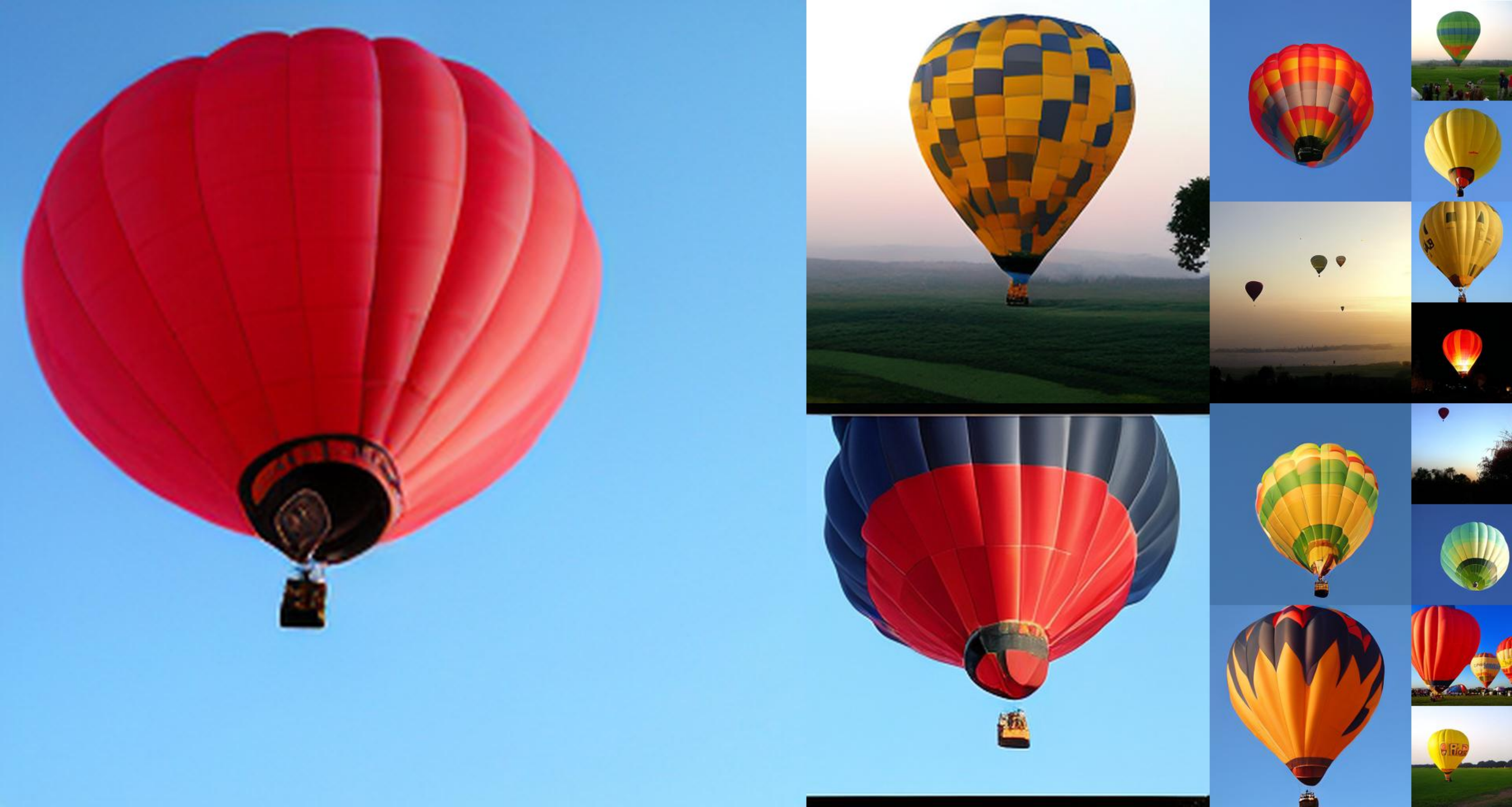}
        \caption{Balloon (417)}
        \label{fig:rec_fm_417}
    \end{subfigure}
    \hfill
    \begin{subfigure}[t]{0.48\linewidth}
        \centering
        \includegraphics[width=\linewidth]{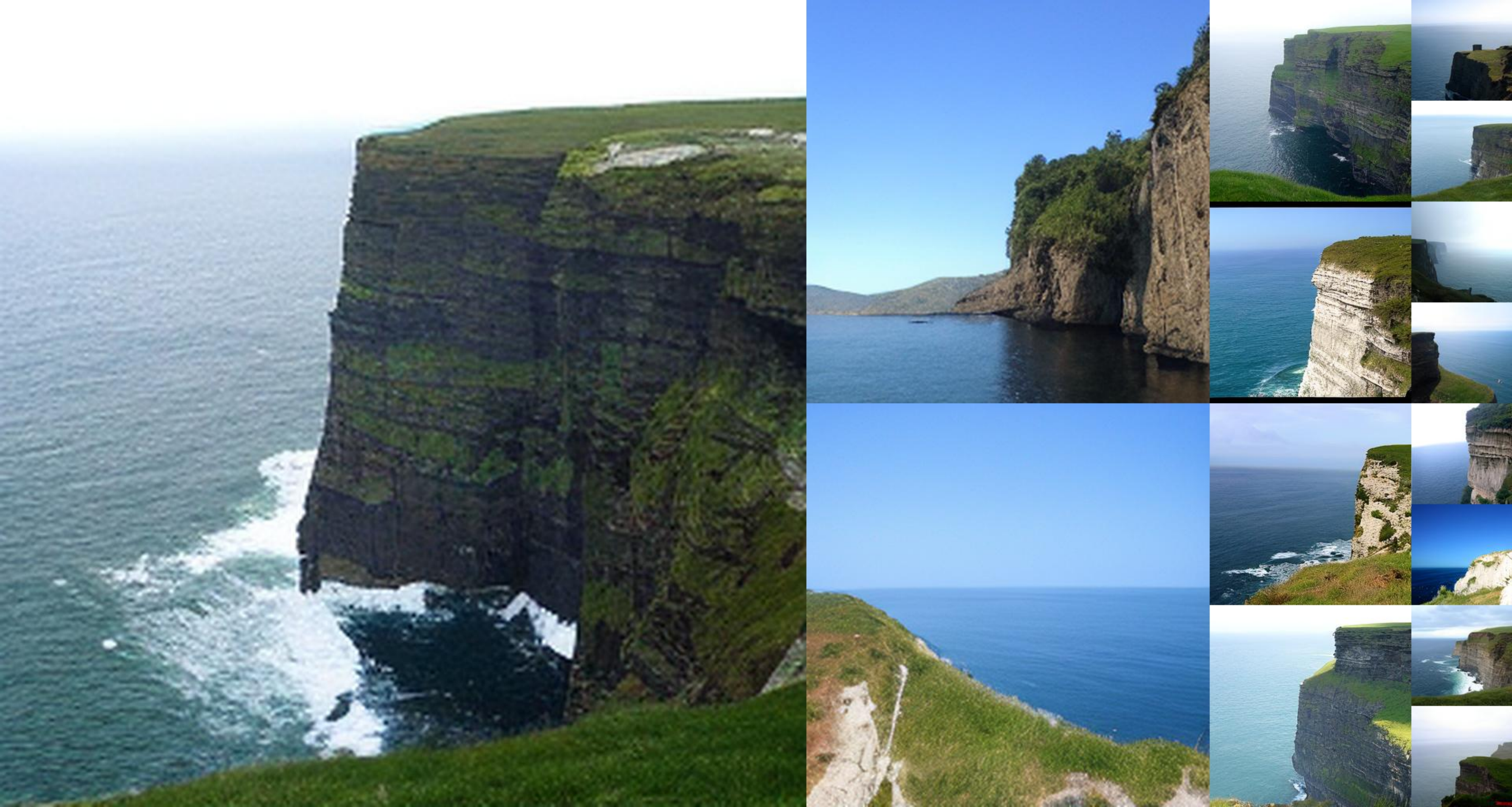}
        \caption{Cliff drop-off (972)}
        \label{fig:rec_fm_972}
    \end{subfigure}

    \caption{\textbf{Uncurated $256\times256$ RecFM-XL samples (Continued).} Each panel shows samples from a different ImageNet class.}
    \label{fig:recfm_samples_continued}
\end{figure}

\end{document}